\documentclass[]{article}
\usepackage{aaai}
\usepackage{times}
\usepackage{helvet}
\usepackage{courier}
\usepackage{graphicx}
\usepackage{amsmath}
\usepackage{fixltx2e}
\usepackage{graphicx}
\usepackage[numbers]{natbib}
\usepackage[normalem]{ulem}
\usepackage{longtable}
\usepackage{tabu}
\usepackage{amsmath}
\usepackage{textcomp}
\usepackage{marvosym}
\usepackage{wasysym}
\usepackage{amssymb}
\usepackage{verbatim}
\usepackage{mathtools}
\usepackage[normalem]{ulem}
\usepackage{paralist}
\usepackage{amsthm}
\newcommand{\Mc}{{\cal M}}

\newtheorem{example}{Example}
\newtheorem{definition}{Definition}
\newtheorem{lemma}{Lemma}
\newtheorem{theo}{Theorem}

\newtheorem{fact}{Fact}
\newtheorem{corol}{Corollary}

\newtheorem{notation}{Notation}

\renewcommand{\arg}{{\rm Arg}}

\newcommand{\lr}{{\bf L}_{\mathcal{R}}}

\newcommand{\jesse}[1]{\textcolor{red}{#1}}
\newcommand{\jessestr}[2][]{\textcolor{red}{\sout{#2}}}
\newcommand{\chr}[2][]{\textcolor{blue}{#2}}
\newcommand{\crem}[2][]{\textcolor{blue}{\sout{#2}}}
\newcommand{\tolstrut}{%
  \vrule height\dimexpr\fontcharht\font`\A+.1ex\relax width 0pt\relax
}
\DeclareRobustCommand{\textoverline}[1]{%
  \ensuremath{\overline{\mbox{\tolstrut#1}}}%
}
\renewcommand{\chr}[2][]{#2} 
\renewcommand{\crem}[2][]{} 
\renewcommand{\jesse}[2][]{#2} 
\renewcommand{\jessestr}[2][]{} 

\usepackage{pgfplots}

\def\addlegendimage{\csname pgfplots@addlegendimage\endcsname}
\pgfkeys{/pgfplots/number in legend/.style={%
        /pgfplots/legend image code/.code={%
            \node at (0.125,-0.0225){#1}; 
        },%
    },
}
\pgfplotsset{
every legend to name picture/.style={west}
}

\newcommand{\begprflst}{\begin{itemize}} 
\newcommand{\enprflst}{\end{itemize}}

\makeatletter
\renewenvironment{proof}[1][\proofname]{\par
  \vspace{-\topsep}
  \pushQED{\qed}%
  \normalfont
  \topsep0pt \partopsep0pt 
  \trivlist
  \item[\hskip\labelsep
        \itshape
    #1\@addpunct{.}]\ignorespaces
}{%
  \popQED\endtrivlist\@endpefalse
  \addvspace{6pt} 
}
\makeatother

\usepackage{tikz}
\usetikzlibrary{shapes,arrows,positioning,backgrounds,fit,calc,shapes.misc}
\title{Relations between assumption-based approaches in nonmonotonic logic and formal argumentation\thanks{The research of the authors was supported by a Sofja Kovalevkaja award of the Alexander von Humboldt-Foundation, funded by the German Ministry for Education and Research.}}
\author{Jesse Heyninck \and Christian Stra{\ss}er\\
Institute of Philosophy II, Ruhr Universit\"at Bochum\\
Universit\"atstra{\ss}e 150\\
44 800 Bochum, Germany\\
}

\begin{document}
\maketitle

\begin{abstract}
\begin{quote}
In this paper we make a contribution to the unification of formal models of defeasible reasoning. We present several translations between formal argumentation frameworks and nonmonotonic logics for reasoning with plausible assumptions. More specifically, we translate adaptive logics into assumption-based argumentation and ASPIC$^+$, ASPIC$^+$ into assumption-based argumentation and a fragment of assumption-based argumentation into adaptive logics. Adaptive logics are closely related to Makinson's default assumptions and \jesse{to a significant class of systems within the tradition of preferential semantics} in the vein of KLM and Shoham. Thus, our results also provide close links between formal argumentation and the latter approaches.
\end{quote}
\end{abstract}


\section{Introduction}
There is a a plenitude of logical approaches to the modelling of defeasible reasoning known as nonmonotonic logics (in short, NMLs). These approaches often use different methods, representational formats or key ideas, making it sometimes difficult to compare them, e.g.\ with respect to the consequence relations they give rise to. Such comparisons are important to systematise the field of NMLs and to gain insights into which forms of defeasible reasoning are expressible in which formal frameworks. An important tool for such comparisons are translations between systems of NML. If one system (or a fragment thereof) is translatable into another system we immediately know that the latter system is at least as expressive as the former. Moreover, this may lead to forms of cross-fertilisation, since meta-theoretic properties become transferable between the translated systems. 

In this contribution we will investigate several such translations. Given the richness of the domain of NMLs, we approach the topic from a specific angle. Our focus will be on structured argumentation, on the one hand, and NMLs that model defeasible inferences in terms of strict inference rules and defeasible assumptions, on the other hand. As a side product, the translation will also cover a significant subclass of NMLs in the KLM paradigm based on preferential semantics \cite{shoham1987,kraus1990}.

At least since Dung introduced abstract argumentation \cite{Dung1995}, formal argumentation has been an important sub-domain of NML. While in abstract argumentation arguments are not phrased in a formal language and the underlying inferences are not explicated, several systems of structured or instantiated formal argumentation have been developed which overcome this limitation (cf.\ \cite{besnard2014introduction} for a partial overview). In this paper we will focus on two of the most prominent accounts: assumption-based argumentation (in short, ABA) \cite{Bondarenko1997,dung2009,Toni2014} and ASPIC$^+$ \cite{Prakken2010,modgil2014}. 

One of the key differences between several formal approaches to defeasible reasoning concerns the question of how to model defeasible inferences. Let $A_1, \ldots, A_n \leadsto B$ denote the defeasible inference from $A_1, \ldots, A_n$ to $B$. The question is whether such an inference should be phrased in terms of a strict inference rule or a defeasible one. A strict inference rule allows for no exceptions: if its premises $A_1, \ldots, A_n$ are true, the consequent $B$ is true as well. In contrast, defeasible rules allow for exceptions, that is, under specific circumstances it may hold that all premises $A_1, \ldots, A_n$ of the rule hold while the consequent $B$ doesn't. Clearly, in the approach with strict rules defeasibility has to enter in a different way. One way is by means of explicitly stated defeasible assumptions ${\rm As}_1, \ldots, {\rm As}_m$, i.e., specific premises which are assumed to hold by default and which can serve as antecedents of strict rules. An inference is retracted in case there is a demonstration that one of the defeasible assumptions ${\rm As}_1, \ldots, {\rm As}_m$ doesn't hold. 

ABA follows the approach based on strict rules and defeasible assumptions. In ASPIC$^+$ both approaches can be represented. Not surprisingly, ABA has been shown to be translatable to ASPIC$^+$ \cite{Prakken2010}. In this paper we will show the other (perhaps more surprising) direction, namely that  ASPIC$^+$ (without priorities) can be translated into ABA and thus that both frameworks are equi-expressive. 

There are several nonmonotonic systems that model defeasible inference by means of strict rules. Among them are adaptive logics (in short, ALs) \cite{universal}, Makinsons' default assumptions and forms of circumscription. Makinson's default assumptions --and in view of the translation in \cite{Vandeputte2013} also ALs-- are a generalisations of approaches based on maximal consistent subsets \cite{rescher1970}. In view of \cite{Amgoud&Besnard:logicallimitsofabstractargumentationframeworks} we know that there are close connections between approaches based on maximal consistent subsets and structured argumentation. In this paper the ties will be strengthened. We show that ALs can be translated into ABA and ASPIC$^+$ and \jesse{present a translation in the opposite direction for a subclass of ABA and ASPIC$^+$.}\jessestr{, for a subclass of ABA and ASPIC$^+$, \chr[2016-03-30]{we also present a translation in the opposite direction}.}

We will proceed as follows: in Sections \ref{sec:adaptive-logics}--\ref{sec:aspic+} we introduce the basic systems (ALs, preferential semantics, default assumptions, ABA, and ASPIC$^+$). In Sections \ref{sec:transl-adapt-logic}--\ref{sec:transl-assumpt-based} we provide translations as indicated in Figure \ref{fig:roadmap}.
\begin{figure}
  \centering
\tikzset{
    >=stealth',
    punkt/.style={
           rectangle,
           rounded corners,
           draw=black, thick,
           text width=6.5em,
           minimum height=2em,
           text centered},
    pil/.style={
           ->,
           thick,
           shorten <=2pt,
           shorten >=2pt,}
}

\begin{tikzpicture}[node distance=1cm, auto]
  \node[punkt] (ABA) {ABA (\S\ref{sec:assumpt-based-argum})};
  \node[below=.8cm of ABA] (dummy1) {};
  \node[punkt, below=.8cm of dummy1] (ASPIC) {ASPIC$^+$ (\S\ref{sec:aspic+})}
  edge[pil,<-,bend left=40] node[auto] {\small\tabular{c}Prakken\\(2010)\endtabular} (ABA.220)
  edge[pil] node[auto] {\S\ref{sec:transl-aspic+-assump}} (ABA.south);
  \node[punkt,right=.5cm of dummy1] (ALs) {ALs (\S\ref{sec:adaptive-logics})}
  edge[pil,bend left=20] node[auto] {\S\ref{sec:transl-adapt-logic}} (ABA.-25)
  edge[pil,<-,bend right=20] node[auto] {\S\ref{sec:transl-assumpt-based}} (ABA.10);
  \node[right=.5cm of ABA] (dummy2) {};
  \node[right=.5cm of ASPIC] (dummy3) {};
  \node[punkt, right=.5cm of dummy2] (KLM) {KLM (\S\ref{sec:KLM:Mak})}
  edge[pil,<->] node[auto] {\S\ref{sec:KLM:Mak}} (ALs.10);
  \node[punkt, right=.5cm of dummy3] (DACR) {DACR (\S\ref{sec:KLM:Mak})}
  edge[pil,<->]  (ALs.east);
  \node[above=.05cm of DACR] {\small ~~~~~~~~~~~~ \tabular{c}Van de Putte\\(2013)\endtabular};
\end{tikzpicture}

  \caption{Roadmap}
  \label{fig:roadmap}
\end{figure}






\section{Adaptive Logics}
\label{sec:adaptive-logics}
ALs are a general framework for the formal explication of defeasible reasoning. It has been applied to a multitude of defeasible reasoning forms \jesse{(mainly related to questions from philosophical logic)}, such as nonmonotonic forms of reasoning with inconsistent information, causal discovery, inductive generalisations, abductive hypothesis generation, normative reasoning, etc. (see \cite[p.86]{strasser:ALDR} \chr[2016-03-31]{for an overview}). 
 
The driving idea behind ALs is to apply defeasible inference rules under explicit normality assumptions. More specifically, given a compact Tarksi logic $\mathbf{L}$ (the \emph{core} or \emph{lower limit logic}) in a formal language $\mathcal{L}$ and with the derivability relation $\vdash_{\mathbf{L}}$, a set of \emph{abnormalities} $\Omega \subseteq \mathcal{L}$ is fixed. Now, whenever the core logic gives rise to $\Gamma \vdash_{\mathbf{L}} A \vee {\sf ab}$ where ${\sf ab} \in \Omega$, $A$ can be derived in the adaptive logic (based on $\mathbf{L}$ and $\Omega$) on the (defeasible) assumption that ${\sf ab}$ is false.\footnote{The disjunction $\vee$ is supposed to be classical. In fact, in the standard format of ALs which we consider here, the core logic is supposed to be supraclassical. Whenever non-classical logics are used as core logics, classical negation $\neg$ and classical disjunction $\vee$ are superimposed.}

\chr[2016-03-31]{In ALs this basic idea of modeling defeasible inferences is implemented in Hilbert-style proofs. We will first explain the proof theory of ALs.\footnote{Due to spatial restrictions we will focus on the main ideas but explain some aspects of the proof theory (such as adaptive strategies) merely in a semi-formal way. For a more thorough explanation the interested reader is referred to \cite{universal,strasser:ALDR}.} Then we give alternative characterizations of the adaptive consequence relations that are central to prove the adequacy of our translations in subsequent sections.} 

\chr[2016-03-31]{In ALs, usual Hilbert-style proofs are adjusted in two major ways. First,} to keep track of normality assumptions, proof lines in adaptive proofs are equipped with an additional column in which the abnormalities are listed that are assumed to be false. \chr[2016-03-31]{Second,} different \emph{retraction mechanisms} for lines with abnormality assumptions that turn out mistaken are implemented in terms of so-called \emph{adaptive strategies}. We will give some examples below. 

\chr[2016-03-31]{To further explain how adaptive proofs work, it is useful to turn to a concrete example. As an illustration,} we take a look at inconsistency-ALs. These are based on paraconsistent core logics such as \textbf{LP} or \textbf{CLuN(s)}\footnote{\textbf{CLuN(s)} is positive classical logic enriched by the law of the excluded middle. For an axiomatization and a semantics see \cite{batens1999inconsistency}.}. These core logics typically do not validate disjunctive syllogism \(A, {\sim} A \vee B \vdash B\) since in case\jessestr{there is a contradiction in} \(A\) \jesse{is involved in a contradiction}, \(B\) would not follow (then \({\sim} A\) would suffice for the disjunction \({\sim} A \vee B\) to be true). Nevertheless, inconsistency-ALs allow for the defeasible application of disjunctive syllogism under the \emph{normality assumption} that there is no contradiction in \(A\). 
Hence, in inconsistency ALs the abnormalities in $\Omega$ typically have the form of a contradiction $A \wedge {\sim} A$.
E.g., in paraconsistent core logics it usually holds that \(A, {\sim} A \vee B \vdash B \vee (A \wedge {\sim} A)\) and thus one can defeasibly derive \(B\) under the assumption that there is no contradiction in \(A\). Clearly, sometimes such assumptions turn out to be mistaken in view of the given premises. Obviously, this is the case \chr[2016-03-31]{if} \(A \wedge {\sim} A\) is derivable from the given premises. A more interesting case is given, if \(A \wedge {\sim} A\) is not directly derivable but \chr[2016-03-31]{it is derivable} as a member of a minimal disjunction of abnormalities. We illustrate this in the following example.
\begin{example}\label{ex1}
Suppose our core logic is a standard paraconsistent logic such as $\mathbf{LP}$ or $\mathbf{CLuN(s)}$. Let
$\Gamma=\{{\sim}p,{\sim}q,  p\lor q, p\lor r, q\lor s\}$.
\end{example}

\begin{tabular}{r l l l}
 1  & ${\sim}p$                        & PREM     & $\emptyset$                \\
 2  & ${\sim}q$                        & PREM     & $\emptyset$                \\
 3  & $p\lor r$                     & PREM     & $\emptyset$                \\
 4  & $q\lor s$                     & PREM     & $\emptyset$                \\
 5  & $p \lor q$      & PREM     & $\emptyset$        \\
 6  & $r \lor (p\land {\sim} p)$           & 1,3,${\bf L}$-Inf     & $\emptyset$        \\
 7  & $s \lor (q\land {\sim} q)$            & 2,4,${\bf L}$-Inf     & $\emptyset$        \\
 8  & $r$                        & 6,RC    & $\{p\land {\sim} p\}$       \\
 9  & $s$                        & 7,RC     & $\{q\land {\sim} q\}$       \\ 
 10 & $r \vee s$                 & 8,${\bf L}$-Inf   & $\{p \wedge {\sim} p\}$\\
 11 & $r \vee s$                 & 9,${\bf L}$-Inf   & $\{q \wedge {\sim}q\}$ \\
 12 & $(p\land {\sim} p)\lor (q\land {\sim} q)$ & 1,2,5,${\bf L}$-Inf & $\emptyset$
\end{tabular} 

Each proof line has 4 elements: a line number, a formula, a justification and a set of abnormalities (which are assumed to be false). All inferences of the core  logic $\mathbf{L}$ can be applied (indicated by $\mathbf{L}$-Inf in lines 6, 10, 11 and 12). In lines 8 and 9 defeasible inferences are made as explained above. E.g., since \jesse{at} line 6 $r \vee (p \wedge {\sim}p)$ is derived, at line 8 the abnormality $p \wedge {\sim}p$ is considered false and thus put in the abnormality column. The rule employed for this is called RC (rule conditional): from $(l; A \vee {\sf ab}; \Delta)$ derive $(l'; A; l,RC; \Delta \cup \{{\sf ab}\})$. When further inferences are made calling upon lines with non-empty sets of abnormalities, these abnormalities are carried over (see lines 10 and 11 \chr[2016-03-31]{where the abnormalities of lines 8 and 9 are carried over}). 

The retraction of lines in adaptive proofs is always \chr[2016-03-31]{determined in view of} the minimal disjunctions of abnormalities derived at a given stage of a proof (on the empty set of abnormalities). At line 12 such a minimal disjunction of abnormalities is derived. Clearly, the abnormalities assumed to be false at lines 8--11 are involved in the given disjunction. \chr[2016-03-31]{There are different retraction mechanisms for ALs: so-called adaptive \emph{strategies}.} According to the \emph{reliability} strategy, any line with an abnormality in the assumption that is part of a minimal disjunction of abnormalities is to be retracted. Retraction is implemented by marking lines that are retracted. In this case: \medskip

\begin{tabular}{r r l l l}
$\checkmark$ & 8  & $r$                        & 6,RC    & $\{p\land {\sim} p\}$       \\
$\checkmark$ & 9  & $s$                        & 7,RC     & $\{q\land {\sim} q\}$       \\ 
$\checkmark$ & 10 & $r \vee s$                 & 8,${\bf L}$-Inf   & $\{p \wedge {\sim} p\}$\\
$\checkmark$ & 11 & $r \vee s$                 & 9,${\bf L}$-Inf   & $\{q \wedge {\sim}q\}$ \\
\end{tabular} \medskip

There are other, less cautious, strategies. For instance, according to the \emph{minimal abnormality strategy}, $r \vee s$ will not be retracted. The reason is as follows. If we interpret our premises strictly \emph{as normal as possible}, then in view of line 12 it will be the case that either $p \wedge {\sim} p$ holds (and $q \wedge {\sim}q$ doesn't), or $q \wedge {\sim}q$ holds (and $p \wedge {\sim}p$ doesn't). In each case, one of the assumptions of line 10 or 11 is warranted. Due to space limitations, we omit the technical details. Yet another strategy is \emph{normal selections}. According to it a line with the set of abnormalities $\Delta$ is retracted (or marked) once $\bigvee \Delta$ is derived on the empty condition.

These retraction mechanisms provided by adaptive strategies make AL proofs dynamic: sometimes a line may get marked, later unmarked, and yet later marked again. In order to define a \emph{consequence relation} we need a stable notion of derivability. It works as follows: a formula at a line $l$ of a proof is \emph{finally derived} at a stage of the proof if $l$ is not marked and every extension of the proof in which it gets marked can be further extended such that it is unmarked again. The consequence relation of ALs is the defined as follows:

\begin{definition}\label{def:cr:al}
  Let  $\mathbf{L}$ be a compact Tarski logic in the formal language $\mathcal{L}$, let $\Omega \subseteq \mathcal{L}$ be a set of abnormalities, and let ${\sf str} \in \{{\sf r}, {\sf ma}, {\sf ns}\}$ be an adaptive strategy (reliability, minimal abnormality, or normal selections). Where $\Gamma \cup \{A\} \subseteq \mathcal{L}$, $\Gamma \vdash_{{\sf str}}^{\Omega,\mathbf{L}} A$ iff $A$ is finally derivable in an adaptive proof from $\Gamma$. 
\end{definition}

For our translations alternative characterisations of the consequence relations defined in terms of final derivability in Definition \ref{def:cr:al} will be very useful. These characterisations are essentially informed by the set of minimal disjunctions of abnormalities derivable from a given premise set by the core logic $\mathbf{L}$.

\begin{definition}\label{sigmadab}
Where \(\Gamma \subseteq \mathcal{L}\): \(\Sigma_{\mathbf{L}}(\Gamma)\) is the set of all non-empty \(\Delta \subseteq \Omega\) such that \(\Gamma \vdash_{\mathbf{L}} \bigvee \Delta\) and for all non-empty \(\Delta' \subset \Delta\), \(\Gamma \nvdash_{\mathbf{L}} \bigvee \Delta'\).
\end{definition}

A choice set over \(\Sigma_{\mathbf{L}}(\Gamma)\) is a set \(\Theta\) for which \(\Delta \cap \Theta \neq \emptyset\) for all \(\Delta \in \Sigma_{\mathbf{L}}(\Gamma)\). 

\begin{definition}
Where \(\Gamma \subseteq \mathcal{L}\): \(\Phi_{\mathbf{L}}(\Gamma)\) is the set of \(\subset\)-minimal choice sets over \(\Sigma_{\mathbf{L}}(\Gamma)\).
\end{definition}

The following facts will be useful in what follows:
\begin{fact}[\cite{strasser:ALDR}]\label{choicesetfundamental}
1. For all choice sets \(\Theta\) over \(\Sigma_{\mathbf{L}}(\Gamma)\) there is a \(\Theta' \in \Phi_{\mathbf{L}}(\Gamma)\) such that \(\Theta' \subseteq \Theta\).

2. \(\phi\in \Phi_{\mathbf{L}}(\Gamma)\) iff \(\phi\) is a choice set of \(\Sigma_{\mathbf{L}}(\Gamma)\) and for all \(A\in \phi\) there is a \(\Delta_A \in \Sigma_{\mathbf{L}}(\Gamma)\) for which \(\{A\} = \Delta_A \cap \phi\).
\end{fact}

We now give representation theorems for all three adaptive strategies, a given core logic $\mathbf{L}$ and a given set of abnormalities $\Omega$.

\begin{theo}[\cite{universal}]
\label{ma}
\(\Gamma \vdash^{\Omega,{\bf L}}_{\sf ma} A\) iff for all \(\Theta \in \Phi_{\mathbf{L}}(\Gamma)\) there is a \(\Delta\subseteq\Omega\setminus \Theta\) such that \(\Gamma \vdash_{\bf L} A\lor \bigvee\Delta\).
\end{theo}

\begin{theo}[\cite{universal}]
\label{rel}
\(\Gamma\vdash^{\Omega,{\bf L}}_{\sf r} A\) iff there is a \(\Delta\subseteq\Omega\setminus \bigcup \Sigma_{\mathbf{L}}(\Gamma)\) such that \(\Gamma\vdash_{\bf L} A\lor \bigvee\Delta\).
\end{theo}

\begin{theo}[\cite{universal}]
\label{ns}
\(\Gamma\vdash^{\Omega,{\bf L}}_{\sf ns} A\) iff there is a \(\Theta \in \Phi_{\mathbf{L}}(\Gamma)\) and a \(\Delta\subseteq\Omega\setminus \Theta\) such that \(\Gamma\vdash_{\bf L} A\lor \bigvee\Delta\).
\end{theo}

\section{Preferential Semantics and Default Assumptions}
\label{sec:KLM:Mak}

The semantics for ALs are a special but rich subclass of the well known preferential semantics as defined in \cite{kraus1990} and \cite{shoham1987}. As in the previous section we assume a core logic $\mathbf{L}$ in a formal language $\mathcal{L}$ and a set of abnormalities $\Omega \subseteq \mathcal{L}$. We also assume that the core logic ${\bf L}$ comes with an adequate model-theoretic semantics and an associated semantic consequence relation $\Vdash_{\mathbf{L}}$. We write $\mathcal{M}(\Gamma)$ for the set of all models of a premise set $\Gamma$. Furthermore, where $M\in \mathcal{M}(\Gamma)$, $Ab(M)=\{A\in \Omega\mid M\models A\}$. A model
$M\in \mathcal{M}(\Gamma)$ is \emph{minimally abnormal} iff there is no $M'\in \mathcal{M}(\Gamma)$ for which $Ab(M') \subset Ab(M)$.
\begin{definition}
\begin{itemize}
\item $\Gamma \Vdash^{\Omega,{\bf L}}_{\sf ma} A$ iff $M\models A$ for every minimally abnormal model of $\Gamma$.
\item $\Gamma \Vdash^{\Omega,{\bf L}}_{\sf r} A$ iff $M\models A$ for every $M\in \mathcal{M}(\Gamma)$ \chr[2016-03-30]{for which all $A \in Ab(M)$} are verified in some minimally abnormal model $M'\in \mathcal{M}(\Gamma)$.
\item \(\Gamma \Vdash^{\Omega,{\bf L}}_{\sf ns} A\) iff there is a minimally abnormal model \(M \in \Mc(\Gamma)\) such that for all \(M' \in \Mc(\Gamma)\) for which \(Ab(M) = Ab(M')\), \(M' \models A\).
\end{itemize}
\end{definition}

ALs in the standard format are sound and complete w.r.t.\ these semantics (proven e.g.\ in  \cite{universal}):
\begin{theo}
Where $\Gamma \cup \{A\} \subseteq \mathcal{L}$ and ${\sf str} \in \{{\sf ma}, {\sf r}, {\sf ns}\}$, $\Gamma \Vdash^{\Omega,{\bf L}}_{\sf str} A$ iff $\Gamma \vdash^{\Omega,{\bf L}}_{\sf str} A$.
\end{theo}

In \cite{Vandeputte2013}, the connection between ALs and Makinson's Default Assumption Consequence Relations (in short, DACRs) \cite[chapter 2]{makinson2005bridges} was established. In \cite[chapter 2]{makinson2005bridges}, it is also shown that many other non-monotonic consequence relations, such as Reiter's Closed World Assumption, Poole's Background Constraints, etc.\ can be expressed as DACRs. DACRs give formal substance to the idea that, in many situations, non-monotonic reasoning  makes use of a set $\Delta$ of defeasible background assumptions in combination with the strict and explicit premises in $\Gamma$. These background assumptions are used to the extent that they are consistent with $\Gamma$. Accordingly, DACRs make use of the notion of maximal consistent subset:
\begin{definition}
Where $\Gamma\cup \Delta\subseteq \mathcal{L}$, $\Theta \subseteq \Delta$ is a maximal $\Gamma$-consistent subset of $\Delta$ iff:
\begin{itemize}
\item $\Gamma\cup \Theta \not\vdash_{\bf L} A$ for some $A \in \mathcal{L}$ and
\item $\Gamma\cup \Theta' \vdash_{\bf L} A$ for all $A \in \mathcal{L}$ and for every $\Theta\subset \Theta' \subseteq \Delta$.
\end{itemize}
${\sf MCS}(\Gamma,\Delta)$ is the set of all maximal $\Gamma$-consistent subsets of $\Delta$.
\end{definition}

\begin{definition}
Where $\Gamma \cup \Delta \cup \{A\} \subseteq \mathcal{L}$,
$\Gamma\vdash^{{\rm DA}, \mathbf{L}}_{\Delta} A$ iff for every $\Delta'\in {\sf MCS}(\Gamma,\Delta)$, $\Gamma\cup \Delta'\vdash_{\bf L} A$.
\end{definition}

The connection between adaptive logic and DACR's is the following:

\begin{theo}\cite[p.10]{Vandeputte2013}
Where $\Gamma \cup \Delta \cup \{A\} \subseteq \mathcal{L}$ and $\Delta^{\neg} = \{\neg B \mid B \in \Delta\}$, $\Gamma \vdash^{{\rm DA},\mathbf{L}}_{\Delta} A$ iff $\Gamma \vdash_{\sf ma}^{\Delta^{\neg},{\bf L}} A$.
\end{theo}




\section{Assumption-Based Argumentation}
\label{sec:assumpt-based-argum}
ABA, thoroughly described in \cite{Bondarenko1997}, is a formal model that allows one to use a set of plausible assumptions ``to extend a given theory'' \cite[p.70]{Bondarenko1997} unless and until there are good arguments for not using such an assumption. 
 
Inferences are implemented in ABA by means of a deductive system consisting of a language and rules formulated over this language:
\begin{definition}[Deductive System]\label{deductivesystems}
A \emph{deductive system} is a pair $(\mathcal{L},\mathcal{R})$ such that
\begin{itemize}
\item $\mathcal{L}$ is a formal language (consisting of countably many sentences).
\item $\mathcal{R}$ is a set of inference rules of the form $A_1, \ldots, A_n \rightarrow A$ and $\; \rightarrow A$, where $A,A_1\ldots,A_n \in \mathcal{L}$
\end{itemize}
\end{definition}
\begin{definition}
An \emph{$\mathcal{R}$-deduction} from a theory $\Gamma$ is a sequence $B_1,\ldots ,B_m$, where $m>0$ such that for all $i=1,\ldots,m$: $B_i\in \Gamma$ or there exists a $A_1, \ldots, A_n \rightarrow B_i \in \mathcal{R}$ such that $A_1,\ldots,A_n\in \{B_1,\ldots, B_{i-1}\}$.
\end{definition}
\begin{definition}
Where $\Gamma \cup \{A\} \subseteq \mathcal{L}$, $\Gamma\vdash_{\mathcal{R}} A$ holds if there is an $\mathcal{R}$-deduction from $\Gamma$ whose last element is $A$.
\end{definition}

We now introduce defeasible assumptions and a contrariness operator to express argumentative attacks. Given a rule system, an assumption-based framework is defined as follows:
\begin{definition}[Assumption-based framework]
An \emph{assumption-based framework} is a tuple ${\bf ABF}=( (\mathcal{L},\mathcal{R}),\Gamma, Ab, {\textoverline{\quad}})$ where:
\begin{itemize}
\item $(\mathcal{L},\mathcal{R})$ is a deductive system
\item $\Gamma\subseteq\mathcal{L}$
\item $\emptyset\neq Ab \subseteq \mathcal{L}$ is the set of candidate assumptions.
\item $\textoverline{\quad}:Ab\rightarrow \mathcal{L}$ is a contrariness operator.\footnote{Note that $\overline{\phantom{A}}$ does \emph{not} denote the set theoretic complement.}
\end{itemize}
\end{definition}

In most structured accounts of argumentation attacks are defined between arguments which are deductions in a given deductive or defeasible system (e.g., in ASPIC$^{+}$, Defeasible Logic Programming \cite{garcia2004defeasible}) or sequents $\Gamma \vdash_{\bf L} A$ where ${\bf L}$ is an underlying core logic (\cite{arieli2015sequent,besnard2001logic}).\footnote{The former are sometimes referred to as \emph{rule-based}  and the latter as \emph{logic-based} systems of argumentation.} In contrast, ABA operates at a higher level of abstraction, since attacks are defined directly on the level of sets of assumptions instead of on the level of $\mathcal{R}$-deductions.\footnote{Some formulations of ABA define attacks on the level of individual arguments. However, since attacks are only possible `on' assumptions, these formulations are equivalent (cf.\ also \cite{Toni2014}).} ABA can thus be viewed as operating on the level of equivalence classes consisting of arguments generated using the same assumptions.  

\begin{definition}[Attacks]
Given an assumption-based framework ${\bf ABF}=( (\mathcal{L},\mathcal{R}),\Gamma, Ab, {\textoverline{\quad}})$:
\begin{itemize}
\item a set of assumptions $\Delta\subseteq Ab$ \emph{attacks} an assumption $A\in Ab$ iff $\Gamma\cup \Delta \vdash_{\mathcal{R}} \textoverline{A}$.
\item a set of assumptions $\Delta\subseteq Ab$ \emph{attacks} a set of assumptions $\Delta'\subseteq Ab$ iff $\Gamma\cup \Delta\vdash_{\mathcal{R}} \textoverline{A}$ for some $A\in\Delta'$.
\end{itemize}
\end{definition} 
 
Consequences of a given assumption-based framework are determined with the use of argumentation semantics. On the basis of argumentative attacks, semantics determine sets of assumptions that are acceptable given different criteria of acceptability, such as the requirement that a given set of assumption should not attack itself, or it should be able to defend itself against attacks by other sets of assumptions. Argumentation semantics have been phrased for abstract frameworks in \cite{Dung1995} and have been generalised to the level of ABA in e.g.\ \cite{Bondarenko1997}.
\begin{definition}[Argumentation semantics]
Where $\Delta \subseteq Ab$:
\begin{itemize}
\item $\Delta$ is \emph{closed} iff $\Delta=\{A\in Ab\mid \Gamma\cup\Delta\vdash_{\mathcal{R}} A\}$. 
\item $\Delta$ is \emph{conflict-free} iff for every $A\in Ab, \Delta\cup \Gamma\not\vdash_{\mathcal{R}} A$ or $\Delta\cup \Gamma\not\vdash_{\mathcal{R}} \overline{A}$.
\item A closed set $\Delta$ is \emph{naive} iff it is maximally (w.r.t.\ set inclusion) conflict-free.
\item A closed set of assumptions $\Delta\subseteq Ab$ is \emph{admissible} iff it is conflict-free and for each closed set of assumptions $\Delta'\subseteq Ab$, if $\Delta'$ attacks $\Delta$, then $\Delta$ attacks $\Delta'$. 
\item A set $\Delta$ is \emph{preferred} iff it is maximally (w.r.t.\ set inclusion) admissible.
\item $\Delta$ is \emph{stable} iff it is closed, conflict-free and attacks every $a\in Ab\setminus \Delta$.
\end{itemize}
We write ${\sf niv}({\bf ABF}), {\sf prf}({\bf ABF})$ resp.\ ${\sf stb}({\bf ABF})$ for the set of naive, preferred resp.\ stable sets of assumptions in ${\bf ABF}$. 
\end{definition}

\begin{example} 
Let $Ab=\{q,\lnot p \lor \lnot q\}$, $\Gamma=\{p\}$, let the rule system $\mathcal{R}$ characterize classical logic and $\overline{A}=\neg A$ (where $\neg$ is classical negation). Then there are two preferred sets: $\{\lnot p\lor \lnot q\}, \{q\}$. To see this note that e.g.\ $\Gamma\cup\{\lnot p\lor \lnot q\}\vdash_{\mathcal{R}} \lnot q$ and $\Gamma\cup\{q\}\vdash_{\mathcal{R}} \lnot (\lnot p\lor \lnot q)$.
\end{example}

We are now in a position to define various consequence relations for ABA:

\begin{definition}
Given an assumption-based framework ${\bf ABF}=( (\mathcal{L},\mathcal{R}),\Gamma, Ab, {\textoverline{\quad}})$ and ${\sf sem} \in \{{\sf niv}, {\sf prf}, {\sf stb}\}$:
\begin{itemize}
\item ${\bf ABF}\vdash^{\cup}_{{\sf sem}} A$ iff $\Gamma\cup \Delta \vdash_{\mathcal{R}} A$ for some $\Delta \in {\sf sem}({\bf ABF})$.
\item ${\bf ABF}\vdash^{\cap}_{{\sf sem}} A$ iff $\Gamma\cup \Delta \vdash_{\mathcal{R}} A$ for every $\Delta \in {\sf sem}({\bf ABF})$.
\item ${\bf ABF}\vdash^{\Cap}_{{\sf sem}} A$ iff $\Gamma \cup \bigcap \{\Delta\mid \Delta \in {\sf sem}\} \vdash_{\mathcal{R}} A$. 
\end{itemize}
\end{definition}


\section{ASPIC$^+$}\label{sec:aspic+}
In ASPIC$^+$, as in ABA, inferences made on the basis of a strict knowledge base can be extended with additional inferences based on plausible assumptions. However, whereas in ABA attacks and extensions where defined directly on the level of these assumptions, in ASPIC$^+$, arguments are specific deductions. 
More precisely, arguments are constructed from a knowledge base using an argumentation system.
An argumentation system is a generalisation of a deductive system (Def.\ \ref{deductivesystems}) that allows for a distinction between strict (i.e.\ deductive or safe) and defeasible rules.\footnote{In the ASPIC$^+$ framework of \cite{Prakken2010}, there is also the possibility to add a preference ordering over the premises and/or defeasible rules. Similar generalisations exist for ALs and approaches based on maximal consistent subsets and their generalisations such as Makinsons' default assumptions. We will present investigations into translations for systems with priorities at a future occasion. In our presentation, we also disregard a special type of premise called `issue' in the context of ASPIC$^+$\jesse{. Issues are} premises that are never acceptable in the sense that they always require further backup by additional arguments. 
}
\begin{definition}[Defeasible Theory]
Given a formal language $\mathcal{L}$, a \emph{defeasible theory} ${\sf R} = (\mathcal{L}, \mathcal{S},\mathcal{D})$ consists of (where $A_1, \ldots, A_n, B \in \mathcal{L}$):
\begin{itemize}
\item a set of strict rules $\mathcal{S}$ of the form $A_1,\ldots,A_n\rightarrow B$  
\item a set of defeasible rules $\mathcal{D}$ of the form $A_1,\ldots,A_n\Rightarrow B$.
\end{itemize}

We also assume there is a naming function $N:\mathcal{S}\cup\mathcal{D}\rightarrow\mathcal{L}$ s.t.\ every rule $r\in \mathcal{S}\cup\mathcal{D}$ gets assigned a unique name.
$A_1,\ldots,A_n$ are called the antecedents and $B$ is called the consequent of $A_1,\ldots,A_n\rightarrow B$
resp.\ $A_1,\ldots,A_n\Rightarrow B$.
\end{definition}

\begin{definition}[Argumentation System]
  Given a defeasible theory $\mathsf{R}$, an \emph{argumentation system} is a tuple $AS=( {\sf R} ,\overline{\phantom{A}})$ where $\overline{\phantom{A}}$ is a contrariness function from $\mathcal{L}$ to $2^{\mathcal{L}}$.
\end{definition}

Arguments are built by using defeasible and/or strict rules to derive conclusions from a knowledge base. A knowledge base consists of 
strict and plausible premises. $\mathcal{K}_{n}$ is the set of all (necessary) axioms, i.e.\ premises that are considered to be outside the reach of argumentative attacks. 
$\mathcal{K}_{a}$ has an analogous function to the defeasible assumptions in ABA: they are deemed \emph{plausible} in that they are assumed to be true unless and until a counterargument is encountered. 

\begin{definition}[Knowledge Base]
  A \emph{Knowledge Base} 
  is a set $\mathcal{K}$, where  $\mathcal{K} = \mathcal{K}_{n} \cup \mathcal{K}_{a}$ and $\mathcal{K}_{n} \cap \mathcal{K}_{a} = \emptyset$.
\end{definition}

\begin{definition}[Arguments]
  Let $AS=( {\sf R} ,\overline{\phantom{A}})$ be an argumentation system and $\mathcal{K} = \mathcal{K}_{a}\cup \mathcal{K}_{n}$ a knowledge base. An \emph{argument} $a$ is one of the following:
\begin{itemize}
\item a \emph{premise argument} $\langle A \rangle$ if $A\in\mathcal{K}$
\item a \emph{strict rule-argument} $\langle a_1,\ldots a_n\mapsto B\rangle$ if $a_1,\ldots a_n$ (with $n\geqslant 0$) are arguments such that there exists a strict rule $\mathrm{conc}(a_1),\ldots \mathrm{conc}(a_n)\rightarrow B\in\mathcal{S}$.
\item a \emph{defeasible rule-argument} $\langle a_1,\ldots a_n \Rrightarrow B\rangle$ if $a_1,\ldots a_n$ (with $n\geqslant 0$) are arguments such that there exists a defeasible rule $\mathrm{conc}(a_1),\ldots \mathrm{conc}(a_n)\Rightarrow B$.
\end{itemize}
\end{definition}
We will use $\arg(AS,\mathcal{K})$ to denote the set of all arguments that can be built from a knowledge base $\mathcal{K}$ using an argumentation system $AS$.

\begin{example}
Let $\mathcal{S}=\{\lnot q\rightarrow \lnot p\}$, $\mathcal{D}=\{ \lnot p \Rightarrow s\}$, $\mathcal{K}_n=\{ \lnot s\}$, and $\mathcal{K}_a=\{\lnot q, \lnot p, q\}$. We have, e.g., the following arguments:

\begin{tabular}{l l l}
$a_1=\langle \lnot q\rangle$ & $a_4=\langle a_3 \Rrightarrow s\rangle$ & $a_7=\langle \lnot s\rangle$  \\
$a_2=\langle \lnot p\rangle$ & $a_5=\langle a_2 \Rrightarrow s\rangle$\\
$a_3=\langle a_1\mapsto \lnot p\rangle$ & $a_6=\langle q\rangle$\\
\end{tabular} 
\end{example}

\begin{definition} Where $a$ is an argument $a=\langle B \rangle$, $a=\langle a_1,\ldots a_n\mapsto B\rangle$ or $a=\langle a_1,\ldots a_n\Rrightarrow B\rangle$, we define:
\begin{itemize}
\item $\mathrm{conc}(a)=B$
\item $\mathrm{sub}(a)=\mathrm{sub}(a_1)\cup\ldots \cup \mathrm{sub}(a_n)\cup \{a\}$
\item where $a$ is a premise argument: $\mathrm{prem}(a)=\{A\}$
\item where $a$ is not a premise argument: $\mathrm{prem}(a)=\{\mathrm{prem}(a')\mid a'\in {\rm sub}(a)\}$.
\end{itemize}
\end{definition}



The distinction between strict and defeasible rule-arguments allows us to define a variety of attack forms:
\begin{definition}[Attacks]\label{def:ASPIC:att}
Where $a,b \in {\rm Arg}(AS, \mathcal{K})$, $a$ attacks $b$ (in signs, $a \rightsquigarrow b$) iff 
\begin{itemize}
\item $\mathrm{conc}(a) \in \overline{B}$ for some $B \in \mathrm{prem}(b) \cap \mathcal{K}_{a}$ (\emph{Undermining}).
\item $\mathrm{conc}(a) \in \overline{B'}$ for some $b'\in {\rm sub}(b)$ such that ${\rm conc}(b')=B'$ and $b'$ is of the form $\langle b'_1,\ldots,b'_n\Rrightarrow B'\rangle$ (\emph{Rebut}).
\item $\mathrm{conc}(a)= \overline{b'}$ for some $b'\in {\rm sub}(b)$ such that $b'$ is a defeasible argument (\emph{Undercut}).
\end{itemize}
\end{definition}

\begin{example}[Ex.\ \ref{ex1}, contd] \jesse{Where $\overline{A}=\{B\mid B\equiv \lnot A\}$ for every $A\in\mathcal{L}$, we have:}
$a_1\rightsquigarrow a_6$, $a_6\rightsquigarrow a_1$, $a_6\rightsquigarrow a_3$, $a_6\rightsquigarrow a_4$, $a_7\rightsquigarrow a_4$, $a_7\rightsquigarrow a_5$.
\end{example}

\begin{definition}[Structured Argumentation Framework]
A \emph{structured argumentation framework} ${\bf AT} = ({\rm Arg}(AS,\mathcal{K}),\rightsquigarrow)$ is a pair where ${\rm Arg}(AS,\mathcal{K})$ is the set of arguments built from $\mathcal{K}$ using the argumentation system $AS$ and $\rightsquigarrow$ is an attack relation over ${\rm Arg}(AS,\mathcal{K})$. 
\end{definition}

Given a structured argumentation framework, we can again make use of Dung's argumentation semantics to define different notions of acceptable sets of arguments.

\begin{definition}[Argumentation Semantics] Given a structured argumentation framework ${\bf AT}=(Arg(AS,\mathcal{K}),\leadsto)$, where $\mathcal{B}\subseteq Arg(AS,\mathcal{K})$,
\begin{itemize}
\item $\mathcal{B}$ is \emph{conflict-free} iff there is no $a,b\in\mathcal{B}$ such that $a\rightsquigarrow b$
\item $\mathcal{B}$ is \emph{naive} iff it is maximally conflict-free.
\item $\mathcal{B}$ \emph{defends} $a\in\mathcal{A}$ iff for every $c\in\mathcal{A}$ for which $c\rightsquigarrow a$, there is a $b\in\mathcal{B}$ such that $b\rightsquigarrow c$.
\item $\mathcal{B}$ is \emph{admissible} iff it is conflict-free and it defends every argument $a\in\mathcal{B}$
\item $\mathcal{B}$ is \emph{preferred} iff it is maximally (w.r.t.\ set inclusion) admissible.
\item $\mathcal{B}$ is \emph{stable} iff it is conflict-free and for every $a\in Arg(AS,\mathcal{K})\setminus \mathcal{B}$, $\mathcal{B}\leadsto a$.
\end{itemize}
We write ${\sf niv}({\bf AT}),{\sf prf}({\bf AT})$ resp.\ ${\sf stb}({\bf AT})$ for the set of naive, preferred resp.\ stable sets of arguments in ${\bf AT}$. 
\end{definition}

\begin{definition} Where ${\bf AT} = (Arg(AS,\mathcal{K}), \leadsto)$ is a structured argumentation framework and ${\sf sem} \in \{{\sf niv}, {\sf prf}, {\sf stb}\}$,
\begin{itemize}
\item ${\bf AT}\vdash^{\cup}_{\sf sem} A$ iff there is an $a\in \mathcal{B}$ with ${\rm conc}(a)=A$ for some $\mathcal{B} \in {\sf sem}({\bf AT})$.
\item ${\bf AT}\vdash^{\cap}_{\sf sem} A$ iff for every $\mathcal{B} \in {\sf sem}({\bf AT})$ there is an $a\in \mathcal{B}$ with ${\rm conc}(a)=A$.
\item ${\bf AT}\vdash^{\Cap}_{{\sf sem}} A$ iff there is an $a\in \mathcal{B}$ with ${\rm conc}(a)=A$ for every $\mathcal{B} \in {\sf sem}({\bf AT})$.
\end{itemize}
\end{definition}

\section{Translating Adaptive Logic to Assumption-Based Argumentation}
\label{sec:transl-adapt-logic}
The idea of the translation from ALs to ABA is the following. We translate the lower limit logic {\bf L} of the given AL into a deductive system, plausible assumptions are negations of abnormalities, and the contrariness operator is classical negation.
Recall that the lower limit logic is a supraclassical Tarski logic. Hence, there are classical negation $\neg$ and classical disjunction $\vee$ in the underlying language of $\mathbf{L}$. In the remainder of this section we will use $\neg$ and $\vee$ denoting these classical connectives.

We now go through the technical details of our translation.
\begin{definition} 
Let ${\bf AL}$ be an AL with the lower limit logic {\bf L} in a formal language $\mathcal{L}$ and the consequence relation $\vdash_{\bf L}$, the set of abnormalities $\Omega \subseteq \mathcal{L}$ and a strategy ${\sf str}$ (reliability, minimal abnormality, or normal selections). Let ${\bf L}$ be characterised by the rules \chr[2016-03-30]{in} $\mathsf{R}$ and the axiom schemes in $\mathsf{A}$. We the define the assumption based framework ${\bf ABF}_{\bf L}^{\Omega}(\Gamma)$ for the premise set $\Gamma \subseteq \mathcal{L}$ as the tuple ${\bf ABF}_{\bf L}^{\Omega}(\Gamma)=( (\mathcal{L}, \mathcal{R}(\mathbf{L})),\Gamma, Ab_{\Omega}, \overline{\phantom{A}})$ where:
\begin{itemize}
\item $\mathcal{R}(\mathbf{L})$ contains all instances of rules in $\mathsf{R}$ and a rule $\rightarrow A$ for all instances $A$ of axiom schemes in $\mathsf{A}$;\footnote{If no axiomatisation of $\mathbf{L}$ is given, we can proceed more brute force and set $\mathcal{R}=\{{A_1,\ldots,A_n} \rightarrow {A}\mid \{A_1,\ldots, A_n\} \vdash_{\bf L} A\}$.}
\item $Ab_{\Omega}=\{\neg A\mid A\in \Omega\}$
\item $\overline{\phantom{A}}: Ab_{\Omega} \rightarrow \mathcal{L}$, where $\overline{\neg A} =A$
\end{itemize}
\end{definition}

Below we show the following representational theorem:
\begin{theo}\label{mainresultABA}
  Where $\Gamma \cup \{A\} \subseteq \mathcal{L}$ and ${\sf sem} \in \{ {\sf niv},{\sf prf},{\sf stb}\}$,
  \begin{enumerate}
  \item ${\bf ABF}_{\bf L}^{\Omega}(\Gamma) \vdash^{\cup}_{\sf sem} A$ iff $\Gamma\vdash^{\Omega,{\bf L}}_{\sf ns} A$
  \item ${\bf ABF}_{\bf L}^{\Omega}(\Gamma) \vdash^{\cap}_{\sf sem} A$ iff $\Gamma\vdash^{\Omega,{\bf L}}_{\sf ma} A$
  \item ${\bf ABF}_{\bf L}^{\Omega}(\Gamma) \vdash^{\Cap}_{\sf sem} A$ iff $\Gamma\vdash^{\Omega,{\bf L}}_{\sf r} A$.
  \end{enumerate}
\end{theo}

To avoid clutter we introduce some notational convention:
\begin{notation}
Where $\Delta \subseteq \Omega$, $\Delta^{\neg} = \{\neg A\mid A\in \Delta\}$ and $\overline{\Delta^{\neg}} = \Delta$.
\end{notation}


The following fact follows immediately in view of the compactness and the transitivity of $\mathbf{L}$.
\begin{fact} \label{fact:ABA:AL:L:R}
Where $\Gamma \cup \{A\}\subseteq\mathcal{L}$, $\Gamma\vdash_{\mathcal{R}(\mathbf{L})} A$ iff $\Gamma\vdash_{\bf L} A$. 
\end{fact}

In view of this fact, we will indiscriminately use $\vdash$ as $\vdash_{\mathcal{R}(\mathbf{L})}$ and $\vdash_{\bf L}$. %
Note that in view of the supraclassicality of $\mathbf{L}$ we have:
\begin{fact}\label{fact:supra:L}
  $\Gamma \cup \Delta^{\neg} \vdash A$ iff $\Gamma \vdash \bigvee \overline{\Delta\jesse{^{\lnot}}} \vee A$.
\end{fact}

We now established that every instantiation of an AL is indeed an assumption-based framework. We prove that the three consequence relations of ALs correspond to intuitive ways of calculating consequences in ABA. The crucial result to prove this is the fact that every preferred extension in some assumption-based framework ${\bf ABF}_{\bf L}^{\Omega}(\Gamma)$ is exactly the set of negations of abnormalities excluding some choice set over the derivable abnormalities. \chr[2016-03-30]{This is shown in the following lemmas.} 

\begin{lemma}\label{ABAconflictfree}
Where $\phi \in \Phi_{\mathbf{L}}(\Gamma)$, $Ab_{\Omega} \setminus \phi^{\neg}$ is stable in ${\bf ABF}_{\bf L}^{\Omega}(\Gamma)$.
\end{lemma}
\begin{proof}
  We first show that $\Delta^{\neg} = Ab_{\Omega} \setminus \phi^{\neg}$ is conflict-free. Assume for a contradiction that it is not and hence that there is a $B \in \Omega$ for which $\Gamma \cup \Delta^{\neg} \vdash B, \neg B$. Hence, by the compactness of $\mathbf{L}$ and Fact \ref{fact:supra:L}, $\Gamma \vdash \bigvee \Theta$ for some finite $\Theta \subseteq \Delta$. Let $\Theta$ be $\subset$-minimal \chr[2016-03-30]{with this property. Hence,} $\Theta \in \Sigma_{\mathbf{L}}(\Gamma)$. However, then $\phi \cap \Theta \neq \emptyset$, a contradiction. 

  We now show that $\Delta^{\neg}$ is stable. For this, let $\neg B \in Ab_{\Omega} \setminus \Delta\jesse{^{\lnot}}$. Hence, $B \in \phi$. With Fact \ref{choicesetfundamental}.2, there is a $\Theta \in \Sigma_{\mathbf{L}}(\Gamma)$ such that $\{B\} = \phi \cap \Theta$. Since $\Gamma \vdash \bigvee \Theta$, by Fact \ref{fact:supra:L} also $\Gamma \cup (\Theta^{\neg} \setminus \{\neg B\}) \vdash B$. By the monotonicity of $\mathbf{L}$, $\Gamma \cup \Delta^{\neg} \vdash B$ which means that $\Delta$ attacks $B$. 

Since $\Delta^{\neg}$ is conflict-free and attacks every \chr[2016-03-31]{$A \in Ab_{\Omega} \setminus \Delta^\neg$}, it is easy to see that $\Delta^{\neg}$ is closed and stable.
\end{proof}

\begin{example}[Ex. \ref{ex1} contd]
Take $Ab_{\Omega}=\{ \lnot( A \land {\sim} A)\mid A\in \mathcal{L}_{\bf CLuN}\}$ and $\mathcal{R}$ an adequate rule system for ${\bf CLuN}$.
Where $\Gamma=\{{\sim} p,{\sim} q, p\lor  q, p\lor r, q\lor s\}$. There are two stable extensions: $Ab_{\Omega}\setminus \{\lnot ( p\land {\sim}p)\}$ and $Ab_{\Omega}\setminus \{\lnot ( q\land {\sim} q))\}$. To see this observe that e.g.\  $\Gamma\cup \{\lnot ( q\land {\sim} q)\} \vdash_{\bf CLuN}  p\land {\sim}p$.
\end{example}

\chr[2016-03-31]{
  \begin{lemma} \label{lem:AL:ABA:sem:eq:2}
   If $\Delta^{\neg} \subseteq Ab_{\Omega}$ is conflict-free in \({\bf ABF}_{\bf L}^{\Omega}(\Gamma)\) then there is a \(\phi \in \Phi_{\bf L}(\Gamma)\) for which \(\Delta \subseteq \Omega \setminus \phi\).
  \end{lemma}
  \begin{proof}
    Suppose \(\Delta \not\subseteq \Omega \setminus \phi\) for all \(\phi \in \Phi_{\bf L}(\Gamma)\) and \(\Delta \subseteq \Omega\). By Fact \ref{choicesetfundamental}, \(\Omega \setminus \Delta\) is not a choice set of \(\Sigma_{\bf L}(\Gamma)\). Thus, there is a \(\Theta \in \Sigma_{\bf L}(\Gamma)\) for which \(\Theta \subseteq \Delta\). Since \(\Gamma \vdash \bigvee\Theta\), also \(\Gamma \cup (\Theta \setminus \lbrace A \rbrace) \vdash \neg A\) for any \(A \in \Theta\). Thus, \(\Gamma \cup \Delta\) is not \({\bf L}\)-consistent since \(\Gamma \cup \Delta \vdash A, \neg A\) by monotonicity. By Fact \ref{fact:ABA:AL:L:R}, \(\Gamma \cup \Delta \vdash_{\mathcal{R}(\mathbf{L})} A, \neg A\) and thus, \(\Delta\) is not conflict-free in \({\bf ABF}_{\bf L}^{\Omega}(\Gamma)\).
  \end{proof}
}

With Lemmas \ref{ABAconflictfree} \chr[2016-03-31]{and \ref{lem:AL:ABA:sem:eq:2}} we immediately get:
\begin{lemma}\label{lem:AL:ABA:sem:eq} 
Where $\Gamma \subseteq \mathcal{L}$,
$\chr[2016-03-31]{\{ Ab_{\Omega} \setminus \phi^{\neg} \mid \phi \in \Phi_{\bf L}(\Gamma) \}} = {\sf stb}(\mathbf{ABF}_{\mathbf{L}}^{\Omega}(\Gamma)) = {\sf prf}(\mathbf{ABF}_{\mathbf{L}}^{\Omega}(\Gamma)) = {\sf niv}(\mathbf{ABF}_{\mathbf{L}}^{\Omega}(\Gamma))$
\end{lemma}

We are now in a position to prove our main result in this section: \medskip

\begin{proof}[Proof of Theorem \ref{mainresultABA}]
In view of Lemma \ref{lem:AL:ABA:sem:eq} it is enough to show the theorem for ${\sf sem}= {\sf stb}$.

Ad 3. ${\bf ABF}_{\bf L}^{\Omega}(\Gamma) \vdash^{\Cap}_{\sf stb} A$ iff $\Gamma \cup \bigcap\{ \Delta \mid \Delta \in {\sf stb}(\mathbf{ABF}_{\mathbf{L}}^{\Omega}(\Gamma))\} \vdash A$. By Lemma \ref{ABAconflictfree}, this is the case iff $\Gamma \cup \bigcap \{(\Omega \setminus \phi)^{\neg} \mid \phi \in \Phi_{\mathbf{L}}(\Gamma)\} \vdash A$. Since $\bigcup \Phi_{\mathbf{L}}(\Gamma) = \bigcup \Sigma_{\mathbf{L}}(\Gamma)$ (which is easy to see and left to the reader), this is equivalent to $\Gamma \cup (\Omega \setminus \bigcup \Sigma_{\mathbf{L}}(\Gamma))^{\neg} \vdash A$. By compactness, monotonicity and Fact \ref{fact:supra:L}, this is equivalent to $\Gamma \vdash A \vee \bigvee \Delta$ for some finite $\Delta \subseteq \Omega \setminus \bigcup \Sigma_{\mathbf{L}}(\Gamma)$. By Theorem \ref{rel} this is equivalent to $\Gamma \vdash^{\Omega,{\bf L}}_{\sf r} A$.


Ad 1.\ and 2.\ Analogous.
\end{proof}

\subsection{Translating Adaptive Logic to ASPIC$^+$}
In \cite{Prakken2010} we have a translation from ABA to ASPIC$^+$. Although this translation requires several assumptions that $\mathbf{ABF}_{\mathbf{L}}^{\Omega}(\Gamma)$ does not satisfy, it turns out that it is easy to prove that any $\mathbf{ABF}_{\mathbf{L}}^{\Omega}(\Gamma)$ can easily be translated to an assumption-based framework that does satisfy these assumptions. 

The underlying idea is basically the same as that for translating AL into ABA: the plausible knowledge base consists of the negated abnormalities, the strict premises of the ASPIC$^+$ framework are the premise set $\Gamma$ and the strict rules of the ASPIC$^+$ framework are the inference rules of the monotonic core logic. Due to spatial restrictions, we are not able to present the full technical details of this translation and the adequacy results here.


\section{Translating ASPIC$^+$ to Assumption-Based Argumentation}
\label{sec:transl-aspic+-assump}
\def\concl{{\rm conc}} \def\Rc{\mathcal{R}} \def\phov{\overline{\phantom{A}}}

In this section we translate ASPIC$^+$ to ABA. Since in ABA we have no defeasible rules and less attack types than in ASPIC$^+$ the possibility of this translation is less expected than the translation in the other direction (as provided in \cite{Prakken2010}). In this section we \jesse{thus} offer an answer to the open question stated in \cite{modgil2014} whether such a translation can be given. Our translation works as follows:

\begin{definition}\label{def:AS:ABF}
  Where \({\rm AS} = (\mathsf{R}, \phov)\) is an argumentation system in the formal language $\mathcal{L}$ with a naming function \(N\) for the rules in $\mathsf{R}$ and \(\mathcal{K} = \mathcal{K}_n \cup \mathcal{K}_a\) is a knowledge base, we translate ${\rm AS}$ into an assumption-based framework \({\bf ABF}({\rm AS}) = ((\mathcal{L}', \mathcal{R}), \mathcal{K}_n, Ab, \phov)\) as follows\footnote{For simplicity, we will assume that the contrariness function of the ASPIC$^+$-framework assigns a unique contrary to every $A\in \mathcal{L}$. If this assumption is not satisfied, one has to add $A^c_1\rightarrow -A,\ldots, A^c_n\rightarrow -A$ for every $A^c_i\in \bar{A}$, where $-A\in\mathcal{L}'\setminus\mathcal{L}$ is the contrary of $A$ in ABA, as suggested by \cite[p.109]{Toni2014}.}
\begin{itemize}
\item $\mathcal{L}' \supseteq \mathcal{L}$ is such that $\mathcal{L}' \setminus \mathcal{L}$ contains for each $r$ in $\mathsf{R}$ a unique name $n(r)$ and its contrary $\overline{n(r)}$;\footnote{Formally: $\mathcal{L}' \setminus \mathcal{L} = \{n(r)\mid r \mbox{ in } {\sf R}\} \cup \{\overline{n(r)} \mid r \mbox{ in } {\sf R}\}$ \chr[2016-03-31]{(where $\{n(r)\mid r \mbox{ in } {\sf R}\} \cap \{\overline{n(r)} \mid r \mbox{ in } {\sf R}\} = \emptyset$)}. This warrants that, unlike the names $N(r) \in \mathcal{L}$ used in ${\rm AS}$, the new names $n(r)$ are not antecedents and consequents of rules in $\mathsf{R}$. We use the new names to 'simulate' defeasible rules in ABA.}
\item \(\mathcal{R}\) contains each strict rule from \(\mathsf{R}\) and for each defeasible rule \(r: A_1, \ldots, A_n \Rightarrow A\) it contains\footnote{We suppose that the rules in $\mathcal{R}$ are instances as opposed to schemes. The translation can easily be adjusted to schemes.}
\begin{itemize}
\item the rule \(n(r), A_1, \ldots, A_n \rightarrow A\) 
\item the rule \(\overline{A} \rightarrow \overline{n(r)}\)
\end{itemize}
\item \(Ab = \mathcal{K}_a \cup \lbrace n(r) \mid r\) is a defeasible rule in \(\mathsf{R} \rbrace\)
\end{itemize}
\end{definition}

Below we will show that the translation is adequate in view of the following corollary:

\def\ABAG{\ensuremath{{\bf ABF}({\rm AS})}}
\begin{corol}\label{aspic2aba}
Where ${\bf AT} = ({\rm Arg}({\rm AS}, \mathcal{K}), \leadsto)$ is a structured argumentation framework and ${\sf sem} \in \{{\sf stb}, {\sf prf}\}$,
\begin{enumerate}
\item \(\ABAG \vdash^{\cup}_{\sf sem} A\) iff \({\bf AT} \vdash^{\cup}_{\sf sem} A\)
\item \(\ABAG \vdash^{\cap}_{\sf sem} A\) iff \({\bf AT} \vdash^{\cap}_{\sf sem} A\).
\item \(\ABAG \vdash^{\Cap}_{\sf sem} A\) iff \({\bf AT} \vdash^{\Cap}_{\sf sem} A\).
\end{enumerate}
\end{corol}

In the following we suppose a given argumentation system ${\rm AS}$ and its translation ${\bf ABF}({\rm AS})$ as in Definition \ref{def:AS:ABF}.

\begin{definition}\label{def:aspic2aba}
Where \(\Delta \subseteq Ab\), \({\rm Arg}_{\Delta} \subseteq {\rm Arg(AS,\mathcal{K})}\) is the set of all arguments \(a\) that use only defeasible assumptions in \(\Delta\), any strict rules, and only defeasible rules \(r\) for which \(n(r) \in \Delta\). 

Where \(\mathcal{A} \subseteq {\rm Arg(AS,\mathcal{K})}\) is a set of arguments, \(Ab_{\mathcal{A}} \subseteq Ab\) is the set of assumptions consisting of (1) defeasible assumptions $A \in \mathcal{K}_{a}$ for which ${\rm prem}(a) = A$ or ${\rm conc}(a) = A$ for some \(a \in \mathcal{A}\) and (2) of \(n(r)\) where \(r\) is a defeasible rule used in some argument in \(\mathcal{A}\). 

Where \(\mathcal{A}\) is a set of arguments in \({\rm Arg}({\rm AS}, \mathcal{K})\), 
\(\mathcal{A}^{\star}\) denotes the set ${\rm Arg}_{Ab_{\mathcal{A}}}$.
\end{definition}
We sometimes write \(Ab_{a}\) instead of \(Ab_{\lbrace a \rbrace}\).

\begin{fact} \label{fact:A:A:star}
  \chr[2016-03-31]{Where \(\mathcal{A} \subseteq {\rm Arg(AS,\mathcal{K})}\) is a set of assumptions, $\mathcal{A} \subseteq \mathcal{A}^{\star}$.}
\end{fact}

\begin{lemma}
\label{lem:ab:concl}
Where \(A \neq n(r)\) for any $r$ in $\mathsf{R}$ and \(\Delta \subseteq Ab\), if \(\mathcal{K}_n \cup \Delta \vdash_{\mathcal{R}} A\) then
\begin{enumerate}
\item if $A \in \mathcal{L}$, there is an \(a \in {\rm Arg}_{\Delta}\) such that \(\concl(a) = {A}\), 
\item else (if \(A = \overline{n(r)}\)), there is an \(a \in {\rm Arg}_{\Delta}\) for which \(\concl(a) = \overline{B}\) where $B$ is the consequent of $r$.
\end{enumerate}
\end{lemma}
\begin{proof}
This can be shown by an induction on the length of a deduction from \(\mathcal{K}_n \cup \Delta \) to \( A\). Base step: this is trivial since \(A \in \mathcal{K}\). Inductive step. We have three possibilities:
\begin{enumerate}
\item \(A\) is the result of applying a strict rule \(r\) in \(\mathsf{R}\) to \(A_1, \ldots, A_n\), or
\item \(A\) is the result of applying \jesse{the translation of a defeasible rule} 
\(r = A_1, \ldots, A_n \Rightarrow A\in\mathsf{R}\) to \(A_{1}, \ldots, A_{n}\) and the rule name \(n(r)\), or
\item $A = \overline{n(r)}$ is the result of applying a rule $\overline{B} \rightarrow \overline{n(r)}$ where $B$ is the consequen\jesse{t} of the defeasible rule $r$ in $\mathsf{R}$.
\end{enumerate}
Ad 2. By the induction hypothesis there are arguments \(a_i\) (\(1\le i \le n\)) s.t.\ \(a_i \in {\rm Arg}_{\Delta}\) and \(\concl(a_i) = A_i\). (Note here that \(A_i \notin \mathcal{L}' \setminus \mathcal{L}\).) Clearly, $a = \langle a_1, \ldots, a_n \Rrightarrow A \rangle \in {\rm Arg}_{\Delta}$ since $n(r) \in \Delta$. %
Ad 1. Analogous. %
Ad 3. By the induction hypothesis and since $\overline{B} \in \mathcal{L}$, there is an argument $a \in {\rm Arg}_{\Delta}$ with ${\rm conc}(a) = \overline{B}$.
\end{proof}

The other direction of Lemma \ref{lem:ab:concl}.1 follows immediately in view of Definition \ref{def:aspic2aba}:
\begin{fact} \label{fact:ab:concl}
Where $\mathcal{A} \subseteq {\rm Arg}({\rm AS},\mathcal{K})$, if there is an $a \in \mathcal{A}$ with ${\rm conc}(a) = A$ then $\mathcal{K}_n \cup Ab_{\mathcal{A}} \vdash_{\mathcal{R}} A$.
\end{fact}

\begin{lemma}
\label{lem:Ac:cl}
Where \(\mathcal{A} \subseteq {\rm Arg}({\rm AS}, \mathcal{K})\), if \(\mathcal{A}\) is admissible then \(\mathcal{A}^{\star}\) is admissible.
\end{lemma}
\begin{proof}
Suppose there are \(a\) and \(b \in \mathcal{A}^{\star}\) s.t.\ \(a\) attacks \(b\). For each attack form it is easy to see that then there is a \(b'' \in \mathcal{A}\) s.t.\ \(a\) attacks \(b''\). Take, for instance, rebuttal. Then \(\concl(a) = \overline{B}'\) where \(B' = \concl(b')\) for some \(b'\in {\rm sub}(b)\). Hence, there is a defeasible rule \(r\) which is applied in \(b'\) to produce \(B'\). By the definition of ${\rm Arg}\jesse{_{Ab_\mathcal{A}}}$
there is an argument \(b'' \in \mathcal{A}\) s.t.\ \(r\) is applied to produce \(\concl(b'') = B'\). For the other attack types (undercuts and undermines) this is shown in an analogous way. 
Now, since \(\mathcal{A}\) is admissible, there is a \(c \in \mathcal{A}\) s.t.\ \(c\) attacks \(a\). Since \chr[2016-03-31]{by Fact \ref{fact:A:A:star}}, \(c \in \mathcal{A}^{\star}\), also \(\mathcal{A}^{\star}\) is defended. To show that $\mathcal{A}^{\star}$ is conflict-free, assume for a contradiction that \(a \in \mathcal{A}^{\star}\). Since $a$ attacks $b'' \in \mathcal{A}$, $\mathcal{A}$ attacks $a$ (\chr[2016-03-31]{due to the admissibility of ${\cal A}$}). However, in view of the fact that $\mathcal{A}$ and $\mathcal{A}^{\star}$ make use of the same defeasible assumptions and defeasible rules and $\mathcal{A}$ attacks $a$ in one of the two, this leads to a selfattack in some argument $a' \in \mathcal{A}$. E.g., suppose $\mathcal{A}$ undermines $a$ in some $B \in {\rm prem}(a)$. Then $B \in Ab_{\mathcal{A}}$. Hence there is an argument $a' \in \mathcal{A}$ with $B \in {\rm prem}(a')$ and $\mathcal{A}$ attacks $a'$. Since $\mathcal{A}$ is conflict-free, this is a contradiction.
\end{proof}

\begin{lemma}\label{lem:A:star:closure}
  Where \(\mathcal{A} = \mathcal{A}^{\star}  \subseteq {\rm Arg}({\rm AS}, \mathcal{K})\), $Ab_{\mathcal{A}}$ is closed.
\end{lemma}
\begin{proof}
  Suppose $\mathcal{A} = \mathcal{A}^{\star}$ and $\mathcal{K}_n \cup Ab_{\mathcal{A}} \vdash_{\mathcal{R}} A$ for some $A \in Ab$. We have two possibilities: (1) $A = n(r)$ for some $r$ in $\mathsf{R}$ or (2) $A \in \mathcal{K}_{a}$. Ad 1. Since there are no rules with consequent $n(r)$, $n(r) \in Ab_{\mathcal{A}}$. Ad 2. By Lemma \ref{lem:ab:concl}, there is an $a \in \mathcal{A}^{\star} = \mathcal{A}$ with ${\rm conc}(a) = A$. Hence, by the definition of $Ab_{\mathcal{A}}$, $A \in Ab_{\mathcal{A}}$.
\end{proof}

\begin{lemma}
\label{lem:aba:aspic:1}
Where \(\mathcal{A} = \mathcal{A}^{\star}  \subseteq {\rm Arg}({\rm AS}, \mathcal{K})\), if \(\mathcal{A}\) is admissible then \(Ab_{\mathcal{A}}\) is admissible.
\end{lemma}
\begin{proof}
Suppose $\mathcal{A} = \mathcal{A}^{\star}$. By Lemma \ref{lem:A:star:closure}, $Ab_{\mathcal{A}}$ is closed.
Suppose \(Ab_{\mathcal{A}}\) is not conflict-free. Hence, \(\mathcal{K}_n \cup Ab_{\mathcal{A}} \vdash_{\mathcal{R}} \overline{A}\) for some \(A \in Ab_{\mathcal{A}}\). We use Lemma \ref{lem:ab:concl} according to which we have two cases.
Case 1: there is an \(a \in \mathcal{A}^{\star}\) s.t.\ \(\concl(a) = \overline{A}\). Since \(\mathcal{A} = \mathcal{A}^{\star}\), \(a \in \mathcal{A}\) and \(\mathcal{A}\) is not conflict-free.
Case 2: \(A = n(r)\) and there is an \(a \in \mathcal{A}\) for which \(\concl(a) = \overline{B}\) where \(B\) is the consequent of \(r\). Since \(n(r) \in Ab_{\mathcal{A}}\), there is an argument \(a' \in \mathcal{A}\) which uses rule \(r\) to produce \(\concl(a') = B\) and which is thus rebut-attacked by \(a\). Again, \(\mathcal{A}\) is not conflict-free.
\chr[2016-03-31]{Thus, we have shown (by contraposition) that if $\mathcal{A}$ is conflict-free then $Ab_{\mathcal{A}}$ is conflict-free.}

Suppose \(\mathcal{A}\) is admissible, $\Delta$ \jesse{is closed and} attacks \(Ab_{\mathcal{A}}\). Hence, \(\mathcal{K}_n \cup \Delta \vdash_{\mathcal{R}} \overline{A}\) for some \(A \in Ab_{\mathcal{A}}\). By Lemma \ref{lem:ab:concl} we have two cases. Case 1: there is an \(a \in {\rm Arg}_{\Delta}\) s.t.\ \(\concl(a) = \overline{A}\). Hence, \(A \neq n(r)\) for any \(r \in \mathsf{R}\). Clearly, \(a\) attacks \(\mathcal{A}\). Since \(\mathcal{A}\) is admissible, there is a \(b \in \mathcal{A}\) s.t.\ \(b\) attacks \(a\). Then \(Ab_b \subseteq Ab_{\mathcal{A}}\) and \(\mathcal{K}_n \cup Ab_b \vdash_{\mathcal{R}} \concl(b)\). Thus, \(Ab_{b}\) attacks \(Ab_a\) and hence \(Ab_{\mathcal{A}}\) attacks \(\Delta\).

Case 2: \(A = n(r)\) and there is an \(a \in {\rm Arg}_{\Delta}\) s.t.\ \(\concl(a) = \overline{B}\) where \(B\) is the consequent of \(r\). In this case there is an \(a' \in \mathcal{A}\) which uses rule \(r\) and hence \(\concl(a') = B\). Since \(\mathcal{A}\) is admissible, there is a \(c \in \mathcal{A}\) that attacks \(a\). But then \(\Delta_c \subseteq Ab_{\mathcal{A}}\) attacks \(\Delta_{a}\) and hence \(Ab_{\mathcal{A}}\) attacks \(\Delta\).
\end{proof}

\begin{lemma}
\label{lem:aba:aspic:2}
If \(\Delta \subseteq Ab\) is admissible, then \({\rm Arg}_{\Delta}\) is admissible.
\end{lemma}
\begin{proof}
Similar to the previous proof.
\end{proof}


\begin{theo}\label{thm:aspic2aba}
\begin{enumerate}
\item If \(\Delta\) is preferred (resp.\ stable) then \({\rm Arg}_{\Delta}\) is preferred (resp.\ stable).
\item If \(\mathcal{A}\) is preferred (resp.\ stable) then $\Delta$ is preferred (resp.\ stable) for some $\Delta \supseteq Ab_{\mathcal{A}}$ for which ${\rm Arg}_{\Delta} = \mathcal{A}$.
\end{enumerate}
\end{theo}
\begin{proof}
Ad.1 Suppose \(\Delta\) is preferred. Then, by Lemma \ref{lem:aba:aspic:2}, \({\rm Arg}_{\Delta}\) is admissible. Suppose there is an \(\mathcal{A}' \supset {\rm Arg}_{\Delta}\) that is admissible, then by Lemma \ref{lem:aba:aspic:1}, also \(Ab_{\mathcal{A}'}\) is admissible. Since \(\Delta \subset Ab_{\mathcal{A}'}\) this is a contradiction. 

Ad.2 Suppose \(\mathcal{A}\) is preferred. By Lemma \ref{lem:Ac:cl} and since trivially $\mathcal{A} \subseteq \mathcal{A}^{\star}$, $\mathcal{A} = \mathcal{A}^{\star}$.
By Lemma \ref{lem:aba:aspic:1}, $Ab_{\mathcal{A}}$ is admissible. Now suppose that there is a $\Delta \supset Ab_{\mathcal{A}}$ that is admissible. Then by Lemma \ref{lem:aba:aspic:2}, ${\rm Arg}_{\Delta}$ is admissible. Clearly $\mathcal{A} \subseteq {\rm Arg}_{\Delta}$. By the maximality of $\mathcal{A}$, $\mathcal{A} = {\rm Arg}_{\Delta}$.

Due to space limitations we omit the proof for stable extensions.
\end{proof}

Corollary \ref{aspic2aba} follows directly with Theorem \ref{thm:aspic2aba}, Lemma \ref{lem:ab:concl} and Fact \ref{fact:ab:concl}.

\section{Translating Assumption-based Argumentation to Adaptive Logic}\label{sec:transl-assumpt-based}
\def\negkl{\overline}
\def\ABAR{\ensuremath{{\bf ABA}_{{\cal R}}^{Ab}(\Gamma)}}
\def\lr{\ensuremath{{\bf L}_{\mathcal{R}}^{3}}}
\def\LR{\ensuremath{\mathcal{L}_{\mathcal{R}}^{3}}}

In this section we will translate a fragment of assumption-based argumentation to adaptive logic. 

In the following we write \ABAR for the assumption-based framework $((\mathcal{L}, \mathcal{R}), \Gamma, Ab, \overline{\phantom{A}})$. 

For our translation we will use some connectives from Kleene's well-known 3-valued logic $\mathbf{K}_3$ (see Table \ref{tab:K3}) and superimpose them on a logic that is characterised by the rules in $\mathcal{R}$. This works as follows.

We define the 3-valued logic $\lr$ semantically in the following way: we superimpose on the language $\mathcal{L}$ the operators ${\sim}$ and $\vee$ (which are supposed to not occur in the alphabet of $\mathcal{L}$) resulting in the set of well-formed formulas $\LR$.  The operators are characterised by the truth tables in Table~\ref{tab:K3}.\footnote{In the terminology of \cite{Urquhart2001}, \chr[2016-03-30]{Our negation $\sim$ corresponds to Bochvar's 'external negation' and $\overline{\phantom{A}}$ corresponds to Kleene's negation in his $\mathbf{K}_{3}$. Our disjunction $\vee$ is Kleene's strong disjunction. The requirement of supraclassicality for ${\bf L}_{\cal R}^3$ to serve as a core logic for an AL is satisfied in view of the $\langle \vee, {\sim} \rangle$-fragment of ${\bf L}_{\cal R}^3$.}}
\begin{table}[h]
 \centering
 $\array{c|c} A & \overline{A} \\ \hline 1 & 0 \\ 0 & 1 \\ u & u \endarray \quad 
\array{c|c} A & {\sim} A \\ \hline 1 & 0 \\ 0 & 1 \\ u & 1 \endarray \quad 
\array{c|ccc} \vee & 1 & 0 & u \\ \hline 1 & 1 & 1 & 1 \\ 0 & 1 & 0 & u \\ u & 1 & u & u \endarray$
  \caption{Truth-tables for $\textoverline{~~~}$, ${\sim}$ and $\vee$.}
  \label{tab:K3}
\end{table}

\begin{definition}
$v: \mathcal{L} \rightarrow \{0,1,u\}$ is a function which respects the truth-table for $\overline{\phantom{A}}$ (i.e., $v(\overline{A}) = 1$ iff $v(A) = 0$, $v(\overline{A}) = 0$ iff $v(A) = 1$, and $v(\overline{A}) = u$ iff $v(A) = u$). The valuation function $v_M: \LR \rightarrow \{0,\jesse{u},1\}$ is defined inductively as follows: 
\begin{enumerate}
\item where $A \in \mathcal{L}$, $v_M(A) = v(A)$
\item $v_M({\sim}A) = 0$ iff $v_M(A) = 1$, and $v_M({\sim}A) = 1$ else.
\item $v_M(A \vee B) = \max(v_M(A),v_M(B))$ where $0 < u < 1$.  
\end{enumerate}
We write $M \models A$ iff $v_{M}(A) = 1$ (so $1$ is the only designated value). We write $\Vdash_{\lr}$ for the resulting consequence relation.
\end{definition}

We now use $\lr$ as a lower limit logic for an adaptive logic with the set of abnormalties:
\begin{notation}
$\Omega_{Ab}^{\sim}=\{{\sim} A\mid A\in Ab\}$.
\end{notation}

We translate the rules of $\mathcal{R}$ as follows: $A_1, \ldots, A_{n} \rightarrow B$ is translated to ${\sim} A_1 \vee \ldots \vee {\sim} A_n \vee B$.
\begin{notation}
  Where $\mathcal{R}$ is a set of rules, we write $\mathcal{R}^{\sim}$ for the set of translated rules.
\end{notation}

Our two main representational results in this section are (to be proven below):

\begin{theo}\label{aba2AL:niv} 
Where $\Gamma \cup \{A\} \subseteq \mathcal{L}$, and ${\sf sem} = {\sf niv}$,
\begin{enumerate}
\item $\ABAR \vdash^{\cup}_{\sf sem} A$ iff $\Gamma \cup \mathcal{R}^{\sim} \Vdash^{\Omega^{\sim}_{Ab},\lr}_{\sf ns} A$
\item $\ABAR \vdash^{\cap}_{\sf sem} A$ iff $\Gamma \cup \mathcal{R}^{\sim} \Vdash^{\Omega^{\sim}_{Ab},\lr}_{\sf ma}A$
\item $\ABAR \vdash^{\Cap}_{\sf sem} A$ iff $\Gamma \cup \mathcal{R}^{\sim} \Vdash^{\Omega^{\sim}_{Ab},\lr}_{\sf r}A$
\end{enumerate}
\end{theo}

We can strengthen our result if we suppose that the rule system based on $\mathcal{R}$ satisfies the following requirement: where $\Gamma\cup\{A\} \subseteq \mathcal{L}$,
\begin{enumerate}[]
\item[EX] Where $\Delta \subseteq Ab$ is naive in \ABAR\ and $A \in Ab \setminus \Delta$, $\Gamma \cup \Delta \vdash_{\mathcal{R}} \overline{A}$.
\end{enumerate}

This criterion ensures that every naive set is stable.

\begin{theo}\label{aba2AL} 
Where $\Gamma \cup \{A\} \subseteq \mathcal{L}$: if \ABAR\ satisfies (EX), items 1--3 in Theorem \ref{aba2AL:niv} hold for ${\sf sem} \in \{{\sf niv}, {\sf prf},{\sf stb}\}$.
\end{theo}

We are now going to prove the two theorems above. The following notation will be convenient to avoid clutter:
\begin{notation}
$\Delta^{\sim}=\{\sim A\mid A\in \Delta\}$.
\end{notation}
The following facts will be useful below:
\begin{fact}\label{fact:neg:sim}\label{fact:ded:lr}
Where $\Gamma \cup \Delta \cup \{A\} \subseteq \LR$,
  (i) $\overline{ A} \Vdash_{\lr} {\sim} A$, (ii) $\Gamma \Vdash_{\lr} \bigvee \Delta^\sim \vee A$ iff $\Gamma \cup \Delta \Vdash_{\lr} A$.
\end{fact}

$\lr$ is obviously a compact Tarski logic.

We say that $\Gamma \subseteq \mathcal{L}$ is $\mathcal{R}$-consistent iff there is no $A$ such that $\Gamma \vdash_{\mathcal{R}} A, \overline{A}$.

\begin{lemma}\label{rthenl}
Where $\Gamma \cup \{A\} \subseteq \mathcal{L}$,
\begin{enumerate}
\item $\Gamma\vdash_{\mathcal{R}}A$ implies $\Gamma \cup \mathcal{R}^{\sim} \Vdash_{\lr}A$
\item if $\Gamma$ is $\mathcal{R}$-consistent, $\Gamma \cup \mathcal{R}^{\sim} \Vdash_{\lr}A$ implies $\Gamma\vdash_{\mathcal{R}}A$.
\end{enumerate}
\end{lemma}
\begin{proof}
\textit{Ad 1.} Simple induction on the number of proof steps. We show the induction step. Let $M$ be a model of $\Gamma \cup \mathcal{R}^{\sim}$. Suppose $A$ follows by means of the application of a rule $A_1, \ldots, A_n \rightarrow B$. By the induction hypothesis, $M \models A_1, \ldots, A_n$. Also, $M \models {\sim} A_1 \vee \ldots \vee {\sim}A_n \vee B$. Hence, with the truth-tables for ${\sim}$ and $\vee$, $M \models B$.

\textit{Ad 2.} Suppose $\Gamma \nvdash_{\mathcal{R}} A$. We now construct a counter-model $M$ of $\Gamma \cup \mathcal{R}^{\sim}$ for $A$ as follows. Let 
\[ v : B \mapsto \left\{ \array{cc} 1 & \Gamma \vdash_{\mathcal{R}} B \\ 0 & \Gamma \vdash_{\mathcal{R}} \overline{B} \\ u & \mbox{else} \endarray \right.
\]
\chr[2016-03-30]{Note that $v(A) \in \{u, 0\}$ and hence $M \not\models A$.} We have to show that $M$ is a model of $\Gamma \cup \mathcal{R}^{\sim}$. Since $\Gamma$ is $\mathcal{R}$-consistent, the definition warrants that the truth-table for $\overline{\phantom{A}}$ is respected by $v$. We thus only have to check whether $M$ verifies all formulas in $\Gamma \cup \mathcal{R}^{\sim}$. As for $\Gamma$ this holds trivially since every $B \in \Gamma$ is such that $\Gamma \vdash_{\mathcal{R}} B$ and thus $v(B) = 1$. Let now $A_1, \ldots, A_n \rightarrow B \in \mathcal{R}$. We have to check whether $M \models \bigl(\bigvee_{i=1}^n{\sim}A_i \bigr) \vee B$. Assume the opposite. Thus $v_M(A_i) = 1$ ($1 \le i \le n$) and $v_M(B)\in \{0,u\}$. But then $\Gamma \vdash_{\mathcal{R}} A_{i}$ ($1 \le i \le n$) and thus $\Gamma \vdash_{\mathcal{R}} B$. Hence, $v_{M}(B) = 1$, a contradiction.
\end{proof}

We say that a $\Gamma$ is $\lr$-consistent, if there is a $A \in \LR$ for which $\Gamma \not\Vdash_{\lr} A$.\footnote{Or equivalently and analogous to the $\mathcal{R}$-consistency: if there is no $A\in \mathcal{L}_{\bf L}$ s.t.\ $\Gamma\not \Vdash_{\lr} A,\lnot A$.}

\begin{lemma}\label{lem:cons}
 Where $\Gamma \subseteq \mathcal{L}$, if $\Gamma$ is $\mathcal{R}$-consistent then $\Gamma \cup \mathcal{R}^{\sim}$ is $\lr$-consistent.
\end{lemma}
\begin{proof}
  Suppose $\Gamma$ is $\mathcal{R}$-consistent. Then $\Gamma \nvdash_{\mathcal{R}} A, \overline{A}$ for any $A \in \mathcal{L}$. By Lemma \ref{rthenl}, also $\Gamma \cup \mathcal{R}^{\sim} \nVdash_{\lr} A, \overline{A}$ for any $A \in \mathcal{L}$.  
\end{proof}


\begin{lemma}\label{lemma:lr:cflfree}
Where $\Gamma \subseteq \mathcal{L}$: $\Delta^{\sim} \subseteq \Omega_{Ab}^{\sim} \setminus \phi$ for some $\phi \in \Phi_{\lr}(\Gamma\cup \mathcal{R}^{\sim})$, if $\Delta$ is conflict-free \crem[2016-03-30]{and closed} in \ABAR.
\end{lemma}
\begin{proof}
  Suppose $\Delta^{\sim} \not\subseteq \Omega_{Ab}^{\sim} \setminus \phi$ for all $\phi \in \Phi_{\lr}(\Gamma \cup \mathcal{R}^{\sim})$ and $\Delta \subseteq {\rm Ab}$.\crem[2016-03-30]{ is closed in \ABAR} By Fact \ref{choicesetfundamental}, $\Omega_{Ab}^{\sim} \setminus \Delta^{\sim}$ is not a choice set of $\Sigma_{\lr}(\Gamma \cup \mathcal{R}^{\sim})$ which means that there is a $\Theta^{\sim} \in \Sigma_{\lr}(\Gamma \cup \mathcal{R}^{\sim})$ such that $\Theta \subseteq \Delta$. Since $\Gamma \cup \mathcal{R}^{\sim} \Vdash_{\lr} \bigvee \Theta^{\sim}$, by Fact \ref{fact:ded:lr} also $\Gamma \cup (\Theta \setminus \{A\}) \cup \mathcal{R}^{\sim} \Vdash_{\lr} {\sim} A$ for any $A \in \Theta$. %
Hence, $\Gamma \cup \Delta \cup \mathcal{R}^{\sim}$ is not $\lr$-consistent since $\Gamma \cup \Delta \cup \mathcal{R}^{\sim} \Vdash_{\lr} A, {\sim} A$. Thus by Lemma \ref{lem:cons}, $\Gamma \cup \Delta$ is not $\mathcal{R}$-consistent \chr[2016-03-30]{and thus} \crem[2016-03-30]{Since $\Delta$ is closed in \ABAR,} $\Delta$ is not conflict-free.
\end{proof}

\begin{lemma}\label{prefthenchoice}
Where $\Gamma \subseteq \mathcal{L}$, $\Delta^{\sim} = \Omega_{Ab}^{\sim} \setminus \phi$ for some $\phi\in \Phi_{\lr}(\Gamma \cup \mathcal{R}^{\sim})$, $\Delta$ is naive in \ABAR. 
\end{lemma}
\begin{proof}
  Suppose $\Delta^{\sim} = \Omega_{Ab}^{\sim} \setminus \phi$ for some $\phi\in \Phi_{\lr}(\Gamma \cup \mathcal{R}^{\sim})$. %
  
We first prove that $\Delta$ is conflict-free. Suppose for a contradiction, there is a $B \in Ab$ such that $\Gamma \cup \Delta \vdash_{\mathcal{R}} B, \overline{B}$. By Lemma \ref{rthenl}, $\Gamma \cup \Delta \cup \mathcal{R}^{\sim} \Vdash_{\lr} B, \overline{B}$. Hence, $\Gamma \cup \Delta \cup \mathcal{R}^{\sim}$ is $\lr$-inconsistent and by Fact \ref{fact:ded:lr} and compactness, $\Gamma \cup \mathcal{R}^{\sim} \Vdash_{\lr} \bigvee \Theta^{\sim}$ for some finite $\Theta \subseteq \Delta$. Let $\Theta$ be $\subset$-minimal with this property, so that $\Theta^{\sim} \in \Sigma_{\lr}(\Gamma \cup \mathcal{R}^{\sim})$. Since $\phi \cap \Theta^{\sim} = \emptyset$ this is a contradiction to the fact that $\phi \in \Phi_{\lr}(\Gamma \cup \mathcal{R}^{\sim})$.

We now show that $\Delta$ is closed. Assume for a contradiction that $\Gamma \cup \Delta \vdash_{\mathcal{R}} A$ for some $A \in Ab \setminus \Delta$. By Lemma \ref{rthenl}, $\Gamma \cup \Delta \cup \mathcal{R}^{\sim} \Vdash_{\lr} A$. Note that ${\sim}A \in \phi$. By \chr[2016-03-30]{Fact} \crem[2016-03-30]{Lemma} \ref{choicesetfundamental}, there is a $\Theta^{\sim} \in \Sigma_{\lr}(\Gamma \cup \mathcal{R}^{\sim})$ for which $\{{\sim}A\} = \phi \cap \Theta^{\sim}$. Since $\Gamma \cup \mathcal{R}^{\sim} \Vdash_{\lr} \bigvee \Theta^{\sim}$, by Fact \ref{fact:ded:lr}, $\Gamma \cup \mathcal{R}^{\sim} \cup \Theta \setminus \{A\} \Vdash_{\lr} {\sim}A$. By the monotonicity of $\lr$, $\Gamma \cup \mathcal{R}^{\sim} \cup \Delta \Vdash_{\lr} {\sim}A$. Thus, $\Gamma \cup \Delta \cup \mathcal{R}^{\sim}$ is not $\lr$-consistent which implies by Lemma \ref{lem:cons} that $\Gamma \cup \Delta$ is not $\mathcal{R}$-consistent. This contradicts the fact that $\Delta$ is conflict-free.
\end{proof}

The following theorem follows immediately in view of Lemma  \ref{lemma:lr:cflfree} and  Lemma \ref{prefthenchoice}:
\begin{theo}\label{thm:max:conff}
Where $\Gamma \subseteq \mathcal{L}$, $\Delta$ is a naive extension of \ABAR\ iff $\Delta^{\sim} = \Omega_{Ab}^{\sim} \setminus \phi$ for some $\phi \in \Phi_{\lr}(\Gamma \cup \mathcal{R}^{\sim})$.
\end{theo}

If we suppose requirement (EX), we can also prove Theorem \ref{aba2AL}.

\begin{lemma}\label{lem:r1:r2:adm}
Where $\Gamma \subseteq \mathcal{L}$, $\mathcal{R}$ satisfies (EX), and $\Delta^{\sim} = \Omega_{Ab}^{\sim}\setminus \phi$ for some $\phi\in \Phi_{\lr}(\Gamma \cup \mathcal{R}^{\sim})$, $\Delta$ is stable in \ABAR. 
\end{lemma}
\begin{proof}
  Suppose $\Delta^{\sim} = \Omega_{Ab}^{\sim} \setminus \phi$ for some $\phi\in \Phi_{\lr}(\Gamma \cup \mathcal{R}^{\sim})$. In view of Lemma \ref{prefthenchoice} we only need to show that $\Delta$ attacks all $B \in Ab \setminus \Delta$. 
Let thus $B \in Ab \setminus \Delta$. \crem[2016-03-30]{Hence, ${\sim}B \in \phi$. By Fact \ref{choicesetfundamental}, there is a $\Theta^{\sim} \in \Sigma_{\lr}(\Gamma \cup \mathcal{R}^{\sim})$ such that $\{{\sim} B\} = \phi \cap \Theta^{\sim}$. By Lemma \ref{lemma:lr:cflfree}, $\Delta \cup \{B\}$ is not conflict-free.} By (EX), $\Gamma \cup \Delta \vdash_{\mathcal{R}} \overline{B}$. Thus, $\Delta$ attacks $B$. 
\end{proof}

The following Corollary follows immediately in view of Theorem \ref{thm:max:conff} and Lemma \ref{lem:r1:r2:adm}.

\begin{corol}
  Where $\mathcal{R}$ satisfies (EX), each naive set is stable in \ABAR.
\end{corol}

In \cite{Bondarenko1997}, the following was defined resp.\ proven:
\begin{definition}
An assumption-based framework is normal iff every naive set of assumptions is stable.
\end{definition}

\begin{theo}
For any normal assumption-based framework, for any set of assumptions $\Delta\subseteq Ab$, 
 $\Delta$ is naive iff 
 $\Delta$ is stable iff
 $\Delta$ is preferred.
\end{theo}

\begin{corol}\label{thm:prefthenchoice}
If an assumption-based framework satisfies (EX), $\Gamma \subseteq \mathcal{L}$, $\Delta$ is a preferred, stable and naive extension of \ABAR\ iff $\Delta^{\sim} = \Omega_{Ab}^{\sim}\setminus \phi$ for some $\phi \in \Phi_{\lr}(\Gamma \cup \mathcal{R}^{\sim})$.
\end{corol}

We are now in a position to prove our two main theorems in this section. \medskip

\begin{proof}[Proof of Theorems \ref{aba2AL:niv} and \ref{aba2AL}]
[Theorem \ref{aba2AL}.1, $\Leftarrow$]: Suppose that $\Gamma \cup \mathcal{R}^{\sim} \Vdash^{\Omega^{\sim}_{Ab},\lr}_{\sf ns} A$. By Theorem \ref{ns}, there is a $\Delta^{\sim} \subseteq \Omega_{Ab}^{\sim}\setminus \phi$ for some $\phi\in \Phi_{\lr}(\Gamma \cup \mathcal{R}^{\sim})$ s.t.\ $\Gamma \cup \mathcal{R}^{\sim} \Vdash_{\lr} A\lor \bigvee \Delta^{\sim}$. By the monotonicity of ${\lr}$ and Fact \ref{fact:ded:lr}, $\Theta \cup \Gamma \cup \mathcal{R}^{\sim} \Vdash_{\lr} A$ where $\Theta^{\sim} = \Omega_{Ab}^{\sim} \setminus \phi$. By Lemma \ref{lem:r1:r2:adm}, $\Theta$ is stable. Thus, $\Gamma \cup \Theta$ is $\mathcal{R}$-consistent. By Lemma \ref{rthenl}, $\ABAR \vdash^{\cup}_{\sf sem} A$.

The other direction and the other cases are shown analogously.
\end{proof}

\section{Conclusion}
\label{sec:conclusion}

In this paper we provided translations between several prominent systems in nonmonotonic logic (see Fig.~\ref{fig:roadmap} for an overview). In this conclusion we discuss some benefits. 

In view of the translation of ALs into ABA we know that ALs can be understood as forms of formal argumentation. In view of the fact that ALs are equi-expressive with the syntactically characterised preferential semantics in Sec.~\ref{sec:KLM:Mak} and Makinson's default assumptions, the same can be said about the latter two frameworks. Since a broad variety of defeasible reasoning forms in a wide range of application contexts have been explicated within the ALs family (see Sec.~\ref{sec:adaptive-logics}), all these reasoning forms are now available in the domain of formal argumentation. This may lead to further refinements. For instance, once embedded in ASPIC$^+$ we gain rich resources to express preferences and priorities.

In view of the other direction, from a subclass of ABA to ALs, we know that this class can be understood in terms of the model-theoretic tools provided by KLM-style preferential semantics or, alternatively, as consistency management in terms of maximal consistent subsets as provided by default assumptions. This also means that meta-theoretic insights from, for instance, ALs become available for this subclass of ABA. For example, the computational complexity of ALs is well-understood \cite{verdee2009adaptive,odintsov2013computability}. Moreover, properties of the AL consequence relations apply to this class of ABA. For instance, we know that adaptive consequence relations are cumulative (in the notation of Section \ref{sec:adaptive-logics}, where ${\sf AL}$ is an adaptive logic, $\Gamma,\Delta, \{A\} \subseteq \mathcal{L}$, and $\Gamma \vdash_{\sf AL} B$ for all $B \in \Delta$, $\Gamma \vdash_{\sf AL} A$ iff $\Gamma \cup \Delta \vdash_{\sf AL}A$). For a study of meta-theoretic properties of ALs see \cite{universal,strasser:ALDR}. Finally, besides the available dialogue-based methods to model argumentative reasoning processes (e.g.\ \cite{dung2006dialectic}), now the dynamic proof theory of adaptive logics can \jesse{also} be used for this purpose. 

Finally, we complete the circle between ABA and ASPIC$^{+}$ (without priorities/preferences) by providing a translation from the latter to the former, whereas the other direction has been presented in \cite{Prakken2010}. As a side-product this provides a way to phrase the defeasible rules of ASPIC$^{+}$ as strict rules. This shows that the strict fragment of ASPIC$^{+}$ (without strict rules and thus without rebuttals and undercuts) is equi-expressive with full ASPIC$^{+}$. Such insights are conceptually interesting and may simplify future meta-theoretic investigations into ASPIC$^{+}$.

In future work we intend to generalise our investigations to approaches with priorities and preferences as provided in ASPIC$^{+}$ and some generalisations of ALs. An interesting question will be, for instance, whether full \jesse{ASPIC$^{+}$} is translatable into lexicographic ALs \cite[ch.5]{strasser:ALDR} or whether the latter can be translated to ABA or ASPIC$^{+}$.

\bibliographystyle{abbrv}     

\begin{thebibliography}{10}

\bibitem{Amgoud&Besnard:logicallimitsofabstractargumentationframeworks}
L.~Amgoud and P.~Besnard.
\newblock Logical limits of abstract argumentation frameworks.
\newblock {\em JANCL}, 23(3):229--267, 2013.

\bibitem{arieli2015sequent}
O.~Arieli and C.~Stra{\ss}er.
\newblock Sequent-based logical argumentation.
\newblock {\em A\&C}, 6(1):73--99, 2015.

\bibitem{batens1999inconsistency}
D.~Batens.
\newblock Inconsistency-adaptive logics.
\newblock {\em Logic at Work, Essays dedicated to the memory of Helena
  Rasiowa}, pages 445--472, 1999.

\bibitem{universal}
D.~Batens.
\newblock A universal logic approach to adaptive logics.
\newblock {\em Logica universalis}, 1(1):221--242, 2007.

\bibitem{besnard2014introduction}
P.~Besnard, A.~Garcia, A.~Hunter, S.~Modgil, H.~Prakken, G.~Simari, and
  F.~Toni.
\newblock Introduction to structured argumentation.
\newblock {\em A\&C}, 5(1):1--4, 2014.

\bibitem{besnard2001logic}
P.~Besnard and A.~Hunter.
\newblock A logic-based theory of deductive arguments.
\newblock {\em AI}, 128(1):203--235, 2001.

\bibitem{Bondarenko1997}
A.~Bondarenko, P.~M. Dung, R.~A. Kowalski, and F.~Toni.
\newblock An abstract, argumentation-theoretic approach to default reasoning.
\newblock {\em AI}, 93(1):63--101, 1997.

\bibitem{Dung1995}
P.~M. Dung.
\newblock On the acceptability of arguments and its fundamental role in
  nonmonotonic reasoning, logic programming and n-person games.
\newblock {\em AI}, 77:321--358, 1995.

\bibitem{dung2006dialectic}
P.~M. Dung, R.~A. Kowalski, and F.~Toni.
\newblock Dialectic proof procedures for assumption-based, admissible
  argumentation.
\newblock {\em AI}, 170(2):114--159, 2006.

\bibitem{dung2009}
P.~M. Dung, R.~A. Kowalski, and F.~Toni.
\newblock Assumption-based argumentation.
\newblock In {\em Argumentation in Artificial Intelligence}, pages 199--218.
  Springer, 2009.

\bibitem{garcia2004defeasible}
A.~J. Garc{\'\i}a and G.~R. Simari.
\newblock Defeasible logic programming: An argumentative approach.
\newblock {\em TPLP}, 4(1+ 2):95--138, 2004.

\bibitem{kraus1990}
S.~Kraus, D.~Lehmann, and M.~Magidor.
\newblock Nonmonotonic reasoning, preferential models and cumulative logics.
\newblock {\em AI}, 44(1):167--207, 1990.

\bibitem{makinson2005bridges}
D.~Makinson.
\newblock {\em Bridges from classical to nonmonotonic logic}.
\newblock College Publications, 2005.

\bibitem{modgil2014}
S.~Modgil and H.~Prakken.
\newblock The aspic+ framework for structured argumentation: a tutorial.
\newblock {\em A\&C}, 5(1):31--62, 2014.

\bibitem{odintsov2013computability}
S.~P. Odintsov and S.~O. Speranski.
\newblock Computability issues for adaptive logics in multi-consequence
  standard format.
\newblock {\em SL}, 101(6):1237--1262, 2013.

\bibitem{Prakken2010}
H.~Prakken.
\newblock An abstract framework for argumentation with structured arguments.
\newblock {\em Argument and Computation}, 1(2):93--124, 2010.

\bibitem{rescher1970}
N.~Rescher and R.~Manor.
\newblock {On inference from inconsistent premisses}.
\newblock {\em Theory Decis}, 1(2):179--217, 1970.

\bibitem{shoham1987}
Y.~Shoham.
\newblock Reasoning about change.
\newblock Technical report, Yale Univ., New Haven, CT (USA), 1987.

\bibitem{strasser:ALDR}
C.~Stra{\ss}er.
\newblock {\em Adaptive Logics for Defeasible Reasoning}.
\newblock Springer, 2014.

\bibitem{Toni2014}
F.~Toni.
\newblock A tutorial on assumption-based argumentation.
\newblock {\em A\&C}, 5(1):89--117, 2014.

\bibitem{Urquhart2001}
A.~Urquhart.
\newblock Basic many-valued logic.
\newblock In {\em Handbook of philosophical logic}, pages 249--295. Springer,
  2001.

\bibitem{Vandeputte2013}
F.~Van De~Putte.
\newblock Default assumptions and selection functions: a generic framework for
  non-monotonic logics.
\newblock In {\em MICAI 2013}, pages 54--67. Springer, 2013.

\bibitem{verdee2009adaptive}
P.~Verd{\'e}e.
\newblock Adaptive logics using the minimal abnormality strategy are
  $\pi^{1}_1$-complex.
\newblock {\em Synthese}, 167(1):93--104, 2009.

\end{thebibliography}
{\fontsize{9.0pt}{10.0pt} \selectfont          
}

\end{document}